%% file: main.tex
\documentclass{article} 
\usepackage{arxiv}

\input{math_commands.tex}

\usepackage[utf8]{inputenc} 
\usepackage[T1]{fontenc}    
\usepackage{hyperref}       
\usepackage{url}            
\usepackage{booktabs}       
\usepackage{amsfonts}       
\usepackage{nicefrac}       
\usepackage{microtype}      
\usepackage{lipsum}
\usepackage{hyperref}
\usepackage{url}
\usepackage[round]{natbib}

\usepackage{graphicx} 
\usepackage{amsfonts,amsmath,amssymb,amsthm}
\usepackage{bbm}
\usepackage{cleveref}
\usepackage{algorithm}
\usepackage[algo2e, vlined, ruled, linesnumbered]{algorithm2e}

\makeatletter
\newtheorem*{rep@theorem}{\rep@title}
\newcommand{\newreptheorem}[2]{%
\newenvironment{rep#1}[1]{%
 \def\rep@title{#2 \ref{##1}}%
 \begin{rep@theorem}}%
 {\end{rep@theorem}}}
\makeatother

\newtheorem{theorem}{Theorem}
\newreptheorem{theorem}{Theorem}

\newtheorem{lemma}{Lemma}
\newtheorem{definition}{Definition}
\newtheorem*{theorem*}{Theorem}

\DeclareMathOperator{\indi}{\mathbbm{1}}
\DeclareMathOperator{\bias}{\mathbf{bias}}

\newcommand{\ip}[1]{\left\langle #1 \right\rangle}

\DeclareMathOperator{\Reg}{\mathrm{Reg}}

\DeclareMathOperator{\BiasFifth}{\mathrm{Bias5}}
\DeclareMathOperator{\BiasForth}{\mathrm{Bias4}}
\DeclareMathOperator{\Tl}{\mathcal{T}^{\ell}}
\DeclareMathOperator{\Te}{\mathcal{T}_e}
\DeclareMathOperator{\Tel}{\mathcal{T}_e^{\ell}}
\DeclareMathOperator{\Tf}{\mathcal{T}^{f}}
\DeclareMathOperator{\Tef}{\mathcal{T}_e^{f}}
\DeclareMathOperator{\His}{\mathcal{H}}

\title{High Probability Bound for Cross-Learning Contextual Bandits with Unknown Context Distributions}

\author{Ruiyuan Huang\\
School of Data Science\\
Fudan University\\
Shanghai, China\\
\texttt{RuiyuanHuang00@gmail.com} \\
\And
Zengfeng Huang\thanks{Corresponding Author}\\
School of Data Science\\
Fudan University\\
Shanghai, China\\
\texttt{huangzf@fudan.edu.cn} \\
}

\begin{document}

\maketitle

\begin{abstract}
Motivated by applications in online bidding and sleeping bandits, we examine the problem of contextual bandits with cross learning, where the learner observes the loss associated with the action across all possible contexts, not just the current round's context. Our focus is on a setting where losses are chosen adversarially, and contexts are sampled i.i.d.\ from a specific distribution. This problem was first studied by \citet{balseiro2019contextual}, who proposed an algorithm that achieves near-optimal regret under the assumption that the context distribution is known in advance. However, this assumption is often unrealistic.
To address this issue, \citet{Zimmert2023} recently proposed a new algorithm that achieves nearly optimal expected regret. It is well-known that expected regret can be significantly weaker than high-probability bounds. In this paper, we present a novel, in-depth analysis of their algorithm and demonstrate that it actually achieves near-optimal regret with \emph{high probability}.
There are steps in the original analysis by \citet{Zimmert2023} that lead only to an expected bound by nature. In our analysis, we introduce several new insights. Specifically, we make extensive use of the weak dependency structure between different epochs, which was overlooked in previous analyses. Additionally, standard martingale inequalities are not directly applicable, so we refine martingale inequalities to complete our analysis.

\end{abstract}

\input{introduction}
\input{problem_stament}
\input{existing_analysis}
\input{main_analysis}
\input{conclusions}

\bibliography{our_reference}
\bibliographystyle{unsrtnat}

\appendix
\input{appendix}

\end{document}

%% file: math_commands.tex
\usepackage{amsmath,amsfonts,bm}

\def\eqref#1{equation~\ref{#1}}

\def\1{\bm{1}}

\DeclareMathAlphabet{\mathsfit}{\encodingdefault}{\sfdefault}{m}{sl}
\SetMathAlphabet{\mathsfit}{bold}{\encodingdefault}{\sfdefault}{bx}{n}

\newcommand{\E}{\mathbb{E}}

\newcommand{\Var}{\mathrm{Var}}

\DeclareMathOperator*{\argmin}{arg\,min}

%% file: introduction.tex
\section{Introduction}
In the contextual bandits problem, a learner repeatedly observes a context, chooses an action, and incurs a loss specific to that action. The goal of the learner is to minimize the cumulative loss over the time horizon. The contextual bandits problem is a fundamental problem in online learning having broad applications in fields like online advertising, personalized recommendations, and clinical trials~\citep{background_Schapire_2010,backgroundref_Kale_2010,background_villar_2015}.

We consider the cross-learning contextual bandits problem. In this setting, the learner not only observes the loss for the current action under the current context, but also observes the loss for the current action under all other contexts. This problem models many interesting scenarios. One such example is the problem of learning to bid in first-price auctions. In this problem the context is the bidder's private value for the item, while the action is the bid. The cross-learning structure comes from the fact that the bidder can deduce the utility of the bid under all contexts (i.e., the utility of the bid under different private valuations for the item). Other examples include multi-armed bandits with exogenous costs, dynamic pricing with variable costs, and learning to play in Bayesian games~\citep{balseiro2019contextual}. 

Technically, the most interesting setting for the cross-learning contextual bandits problem is when the losses are chosen adversarially but the contexts are i.i.d.\ samples from an \textit{unknown} distribution $\nu$. Recently, \citet{Zimmert2023} gave an algorithm achieving nearly optimal $\widetilde{O}(\sqrt{KT})$ expected regret in this scenario.

\citet{Zimmert2023} designed a sophisticated algorithm that operates over multiple epochs to achieve near-optimal regret. A key technique in their analysis is to sidestep high-probability bounds and instead focus on bounding the expected summation to improve their results. As a consequence, their analysis only provides a bound that holds in expectation. It is not immediately clear whether this is due to limitations in the analysis or if the algorithm is inherently suboptimal. In any case, if we aim for a high-probability bound, fundamentally new insights are required.

In this paper, we show that the algorithm indeed achieves nearly optimal $\widetilde{O}(\sqrt{KT})$ regret with high-probability.
The key contribution of our paper is the following theorem. 
\begin{reptheorem}{thm:high_probability_main}[Informal]
    The algorithm in \citet{Zimmert2023}  yields a regret bound of order $\widetilde{O}(\sqrt{KT})$ with high probability for any policy $\pi$.
\end{reptheorem}
In this section we only give the informal version of \Cref{thm:high_probability_main}. The formal version can be found in \Cref{sec:main_analysis}.

\subsection{Technical Overview}
\label{subsec:technical_overview}
Our theorem is built on a new and more in-depth analysis of the algorithm in \citet{Zimmert2023}. This new analysis introduces several new insights. In particular, we exploit the weak dependency structure between different epochs, which was overlooked in previous work. One difficulty of doing so is that standard martingale inequalities are not directly applicable, so we refine martingale inequalities to complete our analysis.

To prepare the readers for our new analysis, we first briefly introduce the algorithm in \citet{Zimmert2023}. The algorithm in \citet{Zimmert2023} is an EXP3-type algorithm. The key novelty in their algorithm is the construction of the loss estimates $\widehat{\ell}$ used in the FTRL subroutine.
Due to some technical problems we detail later, the algorithm decomposes the time horizon into epochs of equal length.
In each epoch $e$, the algorithm first estimates the probability\footnote{For technical reasons, in the actual algorithm, the value $f_e(a)$ actually represents the probability of observing the reward of each arm $a$ in epoch $e+2$. For ease of understanding, here we instead let it represent the probability of observing the reward of each $\operatorname{arm} a$ in each epoch $e$.} $f_e(a)$ of observing the reward of each arm $a$ in epoch $e$ by an estimator $\widehat{f}_e(a)$, which is constructed exclusively from samples in epoch $e-1$.
Note that thanks to the cross-learning structure, the probability of observing the reward of each arm $a$ is independent of the contexts.
The algorithm then constructs the loss estimates as an importance-weighted estimator with $\frac{1}{\widehat{f}_e(a)}$ as the importance weight.

\citet{Zimmert2023} showed that the performance of the algorithm depends on how well the empirical importance weight $\frac{1}{\widehat{f}_{e}(a)}$ concentrates around the expected importance weight $\frac{1}{f_e(a)}$.
Since the estimator $\widehat{f}_e(a)$ is constructed exclusively from samples in a single epoch rather than the entire time horizon, the concentration $|\frac{1}{f_e(a)} - \frac{1}{\widehat{f}_e(a)}|$ is not tight enough.
To achieve the desired $\widetilde{O}(\sqrt{KT})$ regret under a not tight enough concentration, \citet{Zimmert2023} bounds only the expected bias of importance estimator $\E[\frac{1}{f_{e}(a)} - \frac{1}{\widehat{f}_e(a)}]$ rather than providing a high-probability bias bound.
Bounding only the expected bias gives a small enough bound, however, they can achieve a bound only on the expected regret from a bound on the expected bias.

%
We overcome this difficulty and show that their algorithm actually achieves a high-probability bound.
Our key observation is that different epochs in their algorithm  are only weakly dependent on each other.
Thus, the bias $\frac{1}{\widehat{f}_e(a)} - \frac{1}{f_e(a)}$ for each epoch $e$ is also only weakly dependent on each other.
Therefore, although we cannot establish a small enough bound for the bias of a single epoch $\frac{1}{f_e(a)} - \frac{1}{\widehat{f}_e(a)}$, we can give a small enough bound for the cumulative bias across all epochs $\sum_e \frac{1}{f_e(a)} - \frac{1}{\widehat{f}_e(a)}$. 
We then use the bound on the cumulative bias to bound the cumulative regret.

In addition to utilizing the weak dependency structure between different epochs, we also address two further technical difficulties to establish our result. 
The first difficulty is that the existing regret decomposition is too crude to yield a high-probability bound.
\citet{Zimmert2023} establish an $\widetilde{O}(\sqrt{KT})$ expected regret by decomposing the regret into different parts and bounding each part separately. Although their decomposition gives an $\widetilde{O}(\sqrt{KT})$ expected regret bound, it is too crude to derive a tight high-probability regret bound, even after utilizing the weak dependency structure.
We carefully rearrange the regret decomposition to address this difficulty.

Secondly, we cannot simply apply standard martingale concentration inequalities to $\sum_e \frac{1}{f_e(a)} - \frac{1}{\widehat{f}_e(a)}$ to bound its deviation.
The main problem is that the random variable  $\frac{1}{f_e(a)} - \frac{1}{\widehat{f}_e(a)}$ is not almost surely bounded by a constant, which makes standard martingale concentration inequalities inapplicable.
We introduce a surrogate sequence of random variables as a bridge to address this problem.
We bound the sum over the surrogate sequence, and show that the sum over the real sequence is equal to the surrogate sequence with high probability.

\subsection{Related Works}
The cross-learning contextual bandits problem was first proposed in \citet{balseiro2019contextual}. They achieve the nearly optimal $\widetilde{O}(\sqrt{KT})$ regret under two scenarios: (1) when both losses and contexts are stochastic, and (2) when losses are adversarial and contexts are stochastic with a known distribution. When losses are adversarial and contexts are stochastic with an unknown distribution, they only achieve the suboptimal $\widetilde{O}(K^{1/3}T^{2/3})$ regret. More recently, \citet{Zimmert2023} gave a new algorithm that achieves the nearly optimal $\widetilde{O}(\sqrt{KT})$ regret in expectation under adversarial losses and stochastic contexts with an unknown distribution.

An important application of the cross-learning contextual bandits problem, which is also the primary motivation for proposing this problem in \citet{balseiro2019contextual}, is to solve the problem of learning to bid in first-price auctions. 
In this problem the context is the bidder's private value for the item, while the action is the bid. The cross learning structure comes from the fact that the bidder can deduce the utility of the bid under all contexts (i.e., the utility of the bid under different private valuations for the item). 

\citet{balseiro2019contextual} used the cross-learning contextual bandits problem to model the bidding problem and obtained an $O(T^{3/4})$ regret bound for bidders with an unknown value distribution participating in adversarial first-price auctions, where the only feedback is whether the bidder wins the auction. Later, many works studied different settings of the bidding in first price auctions problem. For example, \citet{Han2020Stochastic} considered the problem with censored feedback, where each bidder observes the winning bid. \citet{Han2020Adversarial} considered the scenario when the value is also adversarial. \citet{Aiauction,WangAuction} considered the problem under budget constraints. In all these scenarios, the cross learning structure between different values is an essential component of the analysis.

Another interesting application of the cross-learning contextual bandits problem is the sleeping bandits problem~\citep{Kleinberg2010,Neu2014,Kale16,Saha2020}. In this problem, a certain set of arms is unavailable in each round. The sleeping bandits problem is motivated by instances like  some items might go out of stock in retail stores or on a certain day some websites could be down. When losses are adversarial and availabilities are stochastic, previous work either requires exponential computing time~\citep{Kleinberg2010,Neu2014} or results in suboptimal regret~\citep{Kale16,Saha2020}. The first computationally efficient algorithm with optimal regret $\widetilde{O}(\sqrt{KT})$ is proposed in \citet{Zimmert2023} by modeling the problem as a cross-learning contextual bandit.

We also note that handling unknown context distributions is a common and challenging problem across various contextual bandit problems. For example, in the adversarial linear contextual bandits problem~\citep{Neu2020contextual}, the linear MDP problem~\citep{Dai2023}, and the oracle-based adversarial contextual bandits problem~\citep{Luo2016}, existing algorithms often rely on knowledge of the context distribution. Removing the reliance on knowledge of the context distribution is typically non-trivial~\citep{ZimmertBypass,Dai2023}.

%% file: problem_stament.tex
\section{Problem Statament}

We study a contextual $K$-armed bandit problem over $T$ rounds, with contexts belonging to the set $[C]$. At the beginning of the problem, an oblivious adversary selects a sequence of losses $\ell_{t,c}(a) \in [0,1]$ for every round $t \in[T]$, every context $c \in [C]$, and every arm $a \in[K]$. In each round $t$, we begin by sampling a context $c_t \sim \nu$ i.i.d. from an unknown distribution $\nu$ over $[C]$, and we reveal this context to the learner. Based on this context, the learner selects an arm $a_t \in[K]$ to play. The adversary then reveals the function $\ell_{t,c}(a_t)$, and the learner suffers loss $\ell_{t, c_t}\left(a_t\right)$. Notably, the learner observes the loss for every context $c \in [C]$, but only for the arm $a_t$ they actually played.

We aim to design learning algorithms that minimize regret. Fix a policy $\pi : [C] \rightarrow [K]$. With a slight abuse of notation, we also denote $\pi_c = e_k \in \Delta([K])$ for each $c \in [C]$. The (unexpected) regret with respect to policy $\pi$ is
\[\Reg(\pi) = \sum_{t=1}^T \ell_{t, c_t}(a_t) -  \ell_{t, c_t}(\pi_{c_t}) .\]
We aim to upper bound this quantity (for an arbitrary policy $\pi$).

%
\citet{Zimmert2023} designed an algorithm that achieves an expected regret bound of $\E\left[ \Reg(\pi) \right] \le \widetilde{O}(\sqrt{KT})$ for any policy $\pi$. We will show that the algorithm in \citet{Zimmert2023} actually provides a high-probability regret bound.
%




%% file: existing_analysis.tex
\section{The Algorithm in Schneider and Zimmert (2023)}
In this section, we briefly recap the intuition behind the algorithm proposed in \citet{Zimmert2023} and redescribe the algorithm formally to prepare the readers for our new analysis.

\subsection{Intuition behind Schneider and Zimmert (2023)}

The algorithm proposed in \citet{Zimmert2023} is an EXP3-type algorithm. 
Similar to the well-known EXP3 algorithm, at each round $t$, the algorithm generates a distribution using an FTRL subroutine \[p_{t,c} = \argmin_{p \in \Delta([K])} \left\langle p, \sum_{s=1}^{t-1} \widehat{\ell}_{s,c} \right\rangle - \frac{1}{\eta} F(p)\]
for each context $c$, where $F(p) = \sum_{i=1}^K p_i \log(p_i)$ is the unnormalized negative entropy, $\eta$ is a learning rate, and $\widehat{\ell}$ are loss estimates to be defined later.
The algorithm then essentially samples the action $a_t$ to be played in round $t$ from distribution $p_{t,c_t}$. 

The key novelty in \citet{Zimmert2023} lies in the construction of the loss estimates $\widehat{\ell}$.
An intuitive construction is defined as follows: \[\widetilde{\ell}_{t,c}(a) = \frac{\ell_{t,c}(a)}{\E_{c \sim \nu}[p_{t,c}(a)]} \indi(a_t = a).\]
That is, it uses the classic importance-weighted estimator with $\E_{c \sim \nu}[p_{t,c}(a)]$ as the importance\footnote{In this paper we call terms like $\frac{1}{\E_{c \sim \nu}[p_{t,c}(a)]} $ as the \textit{importance weight} and call terms like $\E_{c \sim \nu}[p_{t,c}(a)]$ as the \textit{importance}.}. 
A straightforward analysis shows that this estimator yields a regret bound of $\widetilde{O}(\sqrt{KT})$.
However, the denominator term $\E_{c \sim \nu}[p_{t,c}(a)]$ is uncomputable because we do not know the distribution of contexts $\nu$.
One may attempt to circumvent this issue by replacing the expected importance $\E_{c \sim \nu}[p_{t,c}(a)]$ with the empirical importance $\frac{1}{t} \sum_{s=1}^t p_{t,c_s}(a)$.
It is not hard to see that whether we achieve the desired  $\widetilde{O}(\sqrt{KT})$ regret depends on how well the empirical importance weight $\frac{1}{\frac{1}{t} \sum_{s=1}^t p_{t,c_s}(a)}$ concentrates around the expected importance weight $\frac{1}{\E_{c \sim \nu}[p_{t,c}(a)]}$.
However, the empirical importance weight $\frac{1}{\frac{1}{t} \sum_{s=1}^t p_{t,c_s}(a)}$ may not concentrate well around the expected importance weight $\frac{1}{\E_{c \sim \nu}[p_{t,c}(a)]}$. 
This is because the probability vector $p_{t,c}$ is not independent of the previous contexts $c_s$, which makes standard concentration inequalities inapplicable.

To address this difficulty, \citet{Zimmert2023} divides the time horizon into epochs of equal length $L$.
At the end of each epoch $e$, the algorithm stores the FTRL distribution at the current time $t = eL$ in a new distribution $s_e$; that is, it takes $s_{e,c}(a) = p_{t,c}(a)$ for each context $c$ and each arm $a$.
The algorithm further decouples the distribution played by the algorithm and the distribution used to estimate the loss vector.
For each time $t$ in epoch $e+2$, the algorithm observes the loss $\ell_{t,c}(a)$ for each arm $a$ and context $c$ with probability $f_{e}(a) \triangleq \E_{c \sim \nu} [s_{e,c}(a)/2]$. 
%
%
The algorithm then estimates the expected importance $f_{e}(a)$ using an empirical importance $\widehat{f}_{e}(a) $ constructing solely from contexts in epoch $e+1$.
%
Finally, the algorithm constructs $\widehat{\ell}_{t,c}(a)$ as an importance-weighted estimator with $\widehat{f}_e(a)$ serving as the importance.

The advantage of their construction is that the empirical importance weight $\frac{1}{\widehat{f}_{e}(a)}$ concentrates around the expected importance weight $\frac{1}{f_e(a)}$ now.
This concentration ensures that the loss estimates $\widehat{\ell}_{t,c}(a)$ are good estimates of the true losses $\ell_{t,c}(a)$.
And this concentration is achieved because the algorithm constructs the estimator using only samples from epoch $e+1$, which are independent of the estimand.

\subsection{A Formal Description of the Algorithm in Schneider and Zimmert (2023)}
In this subsection we describe the algorithm in \citet{Zimmert2023} formally for the sake of completeness. Readers familiar with \citet{Zimmert2023} can skip this subsection safely.

In each round $t$, the algorithm generates a distribution from an FTRL subroutine: \[p_{t,c} = \argmin_{p \in \Delta([K])} \left\langle p, \sum_{s=1}^{t-1} \widehat{\ell}_{s,c} \right\rangle - \frac{1}{\eta} F(p)\]
for each context $c$, where $F(p) = \sum_{i=1}^K p_i \log(p_i)$ is the unnormalized negative entropy, $\eta$ is the learning rate, and $\widehat{\ell}$ are loss estimates to be defined later.
The algorithm will not sample the action $a_t$ played in round t directly from $p_t$ but from a distribution $q_t$ to be defined later.

To construct loss estimates $\widehat{\ell}$, the algorithm divides the time horizon into epochs of equal length $L$. We let $\mathcal{T}_e$ to denote the set of rounds in the $e$-th epoch. At the end of each epoch, the algorithm takes a single snapshot of the underlying FTRL distribution $p_t$ for each context and arm. That is, the algorithm takes
\[s_{e+2,c}(a) = p_{eL,c}(a)\]
and takes
\[s_{1, c}(a) = s_{2, c}(a) = \frac{1}{K}.\]
The index of the snapshot $s_{e+2}$ for epoch $e$ is $e+2$ because we use the snapshot $s_{e+2}$ in epoch $e+2$.

For each round $t \in \Te$, the algorithm observes the loss function of arm $a$ with probability $f_{e}(a)=\mathbb{E}_{c \sim \nu}\left[s_{e,c}(a) / 2\right]$. This is guaranteed by
the following rejection sampling procedure: we first play an arm according to the distribution
\[
q_{t, c_t}= \begin{cases}p_{t, c_t} & \text { if } \forall a \in[K]: p_{t, c_t}(a) \geq s_{e, c_t}(a) / 2 \\ s_{e, c_t} & \text { otherwise.}\end{cases}
\]
After playing arm $a$ according to $q_{t, c_t}$, the learner samples a Bernoulli random variable $S_t$ with probability $\frac{s_{e,c_t}(a)}{2 q_{t,c_t}(a)}$. If $S_t=0$, the learner ignores the feedback from this round; otherwise, they use this loss. 

The only remaining unspecified part is how to construct the loss estimates. We group all
timesteps into consecutive pairs of two. In each pair of consecutive timesteps, we sample from the same distribution and randomly use one to calculate a loss estimate and the other to estimate the sampling frequency. To be precise, let $\mathcal{T}_e^f$ denote the timesteps selected for estimating the sampling frequency and $\mathcal{T}_e^{\ell}$ denote the timesteps used to estimate the losses. Then we define
$$
\widehat{f}_{e}(a)=\frac{1}{\left|\mathcal{T}_{e-1}^f\right|} \sum_{t \in \mathcal{T}_{e-1}^f} \frac{s_{e,c_t}(a)}{2}
$$
which is an unbiased estimator of $f_{e}(a)$. The loss estimators are defined as follows:
$$
\widehat{\ell}_{t,c}(a)=\frac{2 \ell_{t,c}(a)}{\widehat{f}_{e}(a)+\frac{3}{2} \gamma} \indi\left(A_t=a \wedge S_t \wedge t \in \mathcal{T}_e^{\ell}\right)
$$
where $\gamma$ is a confidence parameter to be specified later.

The algorithm is summarized in \Cref{alg:main}. Furthermore, \citet{Zimmert2023} showed that the algorithm achieves an expected regret bound of $\widetilde{O}(\sqrt{KT})$.


\begin{algorithm}
\caption{The algorithm for the cross-learning problem in \citet{Zimmert2023} }
\label{alg:main}
\textbf{Input:} Parameters $\eta, \gamma > 0$ and $L < T$. \\
$\widehat{f}_2\leftarrow 0$\\
\For{$t=1,\dots, L$}{
Observe $c_t$\\
Play $A_t\sim s_{1,c_t}$\\
$\widehat{f}_2 \leftarrow \widehat{f}_2+ \frac{s_{2,c_t}}{2L}$
}
\For{$e=2,\dots,T/L$}{
$\widehat{f}_{e+1}\leftarrow 0$\\
\For{$t=(e-1)L+1,t=(e-1)L+3,\dots,e L-1$}{
Set $p_{t,c}= \argmin_{x\in\Delta([K])} \left(\ip{x,\sum_{s=1}^{t-1}\widehat{\ell}_{s}(c)}-\eta^{-1}F(x)\right)$\\
\For{$t'=t,t+1$}{
    Observe $c_{t'}$\\
    \If{
        $p_{t,c_{t'}}(a) \geq s_{e,c_{t'}}(a)/2$ for all $a \in [K]$
    }{
        Set $q_{t',c_{t'}} = p_{t,c_{t'}}$
    }
    \Else{
        Set $q_{t',c_{t'}} = s_{e,c_{t'}}$
    }
    Play $A_{t'}\sim q_{t',c_{t'}}$ \\
    Observe $\ell_{t',A_{t'}}$
    }
    $t_{f},t_{\ell} \leftarrow \mathsf{RandPerm}(t,t+1)$\\
    $\widehat{f}_{e+1}\leftarrow\widehat{f}_{e+1} + \frac{s_{e+1,c_{t_f}}}{2 (L/2)}$\\
    Sample $S_t\sim \mathcal{B}\left(\frac{s_{e,c_{t_\ell}}(A_{t_\ell})}{2q_{t,c_{t_\ell}}(A_{t_\ell})}\right)$\\
    Set $\widehat{\ell}_{t_\ell,c}(a) = \frac{2\ell_{t_\ell,c}(a)}{\widehat{f}_{e}(a)+\frac{3}{2}\gamma} \mathbb{I}\left(A_t=a, S_t=1\right)$
    }
    $s_{e+2} \leftarrow p_{t}$
}
\end{algorithm}

%% file: main_analysis.tex
\section{Main Result and Analysis}
\label{sec:main_analysis}
The main result of our paper is the following theorem.

\begin{theorem}[Formal]
\label{thm:high_probability_main}
    For any $\delta \in (0,1)$, \Cref{alg:main} with parameters choice $\iota=2 \log (8 K T \frac{1}{\delta})$, $L=\sqrt{\frac{\iota K T}{\log (K)}}=\widetilde{\Theta}(\sqrt{K T \log\frac{1}{\delta}} )$, $\gamma=\frac{16 \iota}{L}=\widetilde{\Theta}(\sqrt{\frac{\log(1/\delta)}{K T}})$, and $\eta=\frac{\gamma}{2(2 L \gamma+\iota)}=\widetilde{\Theta}(1 / \sqrt{K T \log(1/\delta)})$  yields a regret bound of \[\Reg(\pi)= \widetilde{O}\left(\sqrt{KT\log\frac{1}{\delta}} \right)\] with probability at least $1 - \delta$  for any policy $\pi$.
\end{theorem}

In what follows, we briefly overview our proof of \Cref{thm:high_probability_main}. The full proof can be found in the appendix. 





\subsection{Regret Decomposition}
Denote the set of all timesteps used to estimate the frequency as $\Tf$ and denote the set of all timesteps  used to estimate the losses as $\Tl$.  For each $t \in \Te$, we define $\widetilde{\ell}_{t,c}(a)=\frac{2 \ell_{t,c}(a)}{f_{e}(a)+\gamma} \indi\left(A_t=a \wedge S_t \wedge t \in \mathcal{T}_e^{\ell}\right)$. We decompose regret $\Reg(\pi) = \sum_{t=1}^T \ell_{t, c_t}(a_t) -  \ell_{t, c_t}(\pi_{c_t})$ as follows:
\begin{align*}
    \Reg(\pi) &= \underbrace{\sum_{t=1}^T \left( \ell_{t, c_t}(a_t) -  \ell_{t, c_t}(\pi_{c_t}) \right) - 2 \sum_{t \in \Tl} \left( \ell_{t, c_t} (a_t) -  \ell_{t, c_t}(\pi_{c_{t}}) \right)}_{\bias_1} \\
    +& \underbrace{2 \sum_{t \in \Tl}  \left( \ell_{t, c_{t}}(a_{t}) -  \ell_{t, c_{t}}(\pi_{c_{t}})  -   \sum_c \Pr(c)  \langle p_{t,c} - \pi_c , \ell_{t,c} \rangle \right)}_{\bias_2} \\
     +& \underbrace{2 \sum_{t \in \Tl}  \sum_{c} \Pr(c) \langle p_{t,c} - \pi_c, \widehat{\ell}_{t,c} \rangle}_{\textbf{ftrl}} + \underbrace{2 \sum_{t \in \Tl}  \sum_{c} \Pr(c) \left\langle p_{t,c}, \ell_{t,c} - \widetilde{\ell}_{t,c} 
     \right\rangle}_{\bias_3}\\
     +& \underbrace{2 \sum_{t \in \Tl}  \sum_{c} \Pr(c) \left\langle p_{t,c}, \widetilde{\ell}_{t,c} - \widehat{\ell}_{t,c} \right\rangle}_{\bias_4} + \underbrace{2 \sum_{t \in \Tl}  \sum_{c} \Pr(c) \left\langle \pi_c, \widehat{\ell}_{t,c} - \ell_{t,c} \right\rangle}_{\bias_5}.
\end{align*}
%

In our decomposition, the $\bias_1$ term refers to the bias introduced by replacing the regret over the entire time horizon with that over $\Tl$, and the $\bias_2$ term refers to the bias introduced by replacing regret with its linearization. Both of these terms are not hard to bound using standard concentration inequalities. 
Furthermore, the $\textbf{ftrl}$ and $\bias_3$ terms are standard in the analysis of high-probability bounds for bandit algorithms. These two terms are not hard to bound using techniqes from EXP3-IX~\citep{EXP3-IX, Zimmert2023}. 
The $\bias_4$ and $\bias_5$ terms correspond to the bias introduced by constructing the importance estimator $\widehat{f}_e(a)$. These two terms are the terms of interest to bound.

Our decomposition is different from the decomposition in \citet{Zimmert2023}. This difference is essential for deriving a high-probability bound. The key difference lies in the  $\bias_5$ term here. This term saves a $ \sum_{t \in \Tl}  \sum_{c} \Pr(c) \left\langle \pi_c, \widetilde{\ell}_{t,c} - \ell_{t,c} \right\rangle$ term from the $\bias_2$  term in the decomposition given by \citet{Zimmert2023},  which is crucial for deriving a high-probability bound.

\subsection{Identifying a Prototypical Term}
The terms of interest to bound are $\bias_4$ and $\bias_5$. These two terms can be bounded using similar methods. Here we take the $\bias_5$ term as a prototypical term and give a sketch of its analysis. Details can be found in the appendix.

To bound $\bias_5$, we define a filtration $\{\His_t\}_t$ such that the $\sigma$-algebra 
$\His_t$ for each time step $t$ is generated by all randomness before time $t$. Next, we decompose $\bias_5$  as
\begin{align*}
&\sum_{t \in \Tl} \sum_{c} \Pr(c) \left\langle \pi_c, \widehat{\ell}_{t,c} - \ell_{t,c} \right\rangle\\
=& \sum_{t \in \Tl} \sum_{c} \Pr(c) \left( \widehat{\ell}_{t,c}(\pi_c) - \ell_{t,c}(\pi_c) \right)\\
=&  \sum_{t \in \Tl} \sum_{c} \Pr(c) \left( \E\left[ \widehat{\ell}_{t,c}(\pi_c) \big| \His_{t-1} \right]- \ell_{t,c}(\pi_c) \right) \\
&+ \sum_{t \in \Tl} \sum_{c} \Pr(c) \left( \widehat{\ell}_{t,c}(\pi_c) -  \E\left[ \widehat{\ell}_{t,c}(\pi_c) \big| \His_{t-1} \right] \right).
\end{align*}

In this decomposition, the two terms correspond to different components of the bias of the estimator $\widehat{\ell}_{t,c}$. The first term $\sum_{t \in \Tl} \sum_{c} \Pr(c) \left( \E\left[ \widehat{\ell}_{t,c}(\pi_c) \big| \His_{t-1} \right]- \ell_{t,c}(\pi_c) \right) $ corresponds to the bias introduced by constructing the importance estimator $\widehat{f}_e(a)$. The second term $\sum_{t \in \Tl} \sum_{c} \Pr(c) \left( \widehat{\ell}_{t,c}(\pi_c) -  \E\left[ \widehat{\ell}_{t,c}(\pi_c) \big| \His_{t-1} \right] \right)$ corresponds to the bias introduced from the randomness in sampling $a_t$ from $q_t$. 
Once again, the analyses of these two terms follow the same principle. We take the first term \[\sum_{t \in \Tl} \sum_{c} \Pr(c) \left( \E\left[ \widehat{\ell}_{t,c}(\pi_c) \big| \His_{t-1} \right]- \ell_{t,c}(\pi_c) \right) \] as a prototypical term and give a sketch of its analysis for the sake of simplicity. Details can be found in the appendix.

\subsection{Analysis of the Prototypical Term}
To bound the term $\sum_{t \in \Tl} \sum_{c} \Pr(c) \left( \E\left[ \widehat{\ell}_{t,c}(\pi_c) \big| \His_{t-1} \right]- \ell_{t,c}(\pi_c) \right)$, we will use the key observation mentioned at \Cref{subsec:technical_overview}: different epochs in \Cref{alg:main}  are only weakly dependent on each other. To use  this observation rigorously, we introduce an important technical tool.
With a slight abuse of notation, we define a filtration $\left\{ \His_e \right\}_e$, in which for each epoch $e$, the $\sigma$-algebra $\His_e$ is generated by all randomness in epochs $1,\dots,e-1$ and the randomness in $\Tel$. That is, the $\sigma$-algebra $\His_e$ is generated precisely by the context $c_t$, the random seed used in sampling $a_t \sim q_{t,c_t}$, and the random seed used in sampling $S_t\sim \mathcal{B}\left(\frac{s_{e,c_t}(a_t)}{2q_{t,c_t}(a_t)}\right)$ for $t \le (e-1)L$ and $t \in \Tel$.
Note that for each epoch $e$, the  $\sigma$-algebra $\His_e$ excludes the randomness in $\Tef$. This exclusion is crucial for characterizing the weak dependence structure between epochs. 

Given this filtration, we consider the cumulative bias in each epoch. For each epoch $e$, we define a random variable 
\[\BiasFifth_e \triangleq \sum_{t \in \Tel} \sum_{c} \Pr(c)  \left( \E\left[ \widehat{\ell}_{t,c}(\pi_c) \big| \His_{t-1} \right]- \ell_{t,c}(\pi_c) \right).\] 
Then the prototypical term  $\sum_{t \in \Tl} \sum_{c} \Pr(c) \left( \E\left[ \widehat{\ell}_{t,c}(\pi_c) \big| \His_{t-1} \right]- \ell_{t,c}(\pi_c) \right)$ is exactly $\sum_e \BiasFifth_e$. 
Our key observation is that, not only \[\E\left[\sum_{t \in \Tl} \sum_{c} \Pr(c) \left( \E\left[ \widehat{\ell}_{t,c}(\pi_c) \big| \His_{t-1} \right]- \ell_{t,c}(\pi_c) \right)\right] \le 0 \] as shown in ~\citet{Zimmert2023}, but also \[\sum_e \E\left[ \BiasFifth_e | \His_e \right] \sim - \sum_e \frac{\gamma}{f_e(\pi_c) + \gamma}.\]
This key observation improves the inequality in ~\citet{Zimmert2023} in two ways.
Firstly, our bound holds for conditional expectations across epochs, which opens the door to applying martingale concentration inequalities across epochs.
Secondly, our new decomposition improves the upper bound from $0$ to $- \sum_e \frac{\gamma}{f_e(\pi_c) + \gamma}$. This improvement is essential for deriving a high probability bound.
%

Given the new bound $\sum_e \E\left[ \BiasFifth_e | \His_e \right] \sim - \sum_e \frac{\gamma}{f_e(\pi_c) + \gamma}$, we only need to bound the deviation $\sum_e \BiasFifth_e - \E\left[ \BiasFifth_e | \His_e \right]$ to get an upper bound on $\sum_e \BiasFifth_e$.
However, we cannot directly apply standard martingale concentration inequalities to $\sum_e \BiasFifth_e - \E\left[ \BiasFifth_e | \His_e \right]$. 
This is because we need to assume that the random variable $|\BiasFifth_e| \le 2L$  almost surely to get a tight enough concentration bound when applying standard martingale concentration inequalities. 
However, this is not the case. The random variable  $\BiasFifth_e$ exceeds the constant $2L$ with a small but positive probability. This unboundness prevents us from getting a tight enough concentration bound when applying standard martingale concentration inequalities.

To overcome this problem, we consider the indicator function
\[F_e \triangleq \indi\left( \forall a, \left|\widehat{f}_{e}(a)-f_{e}(a)\right| \leq 2 \max \left\{\sqrt{\frac{f_{e}(a)\iota}{L}}, \frac{\iota}{L}\right\} \right)\] defined in ~\citet{Zimmert2023}. 
We show that we also have $\sum_e \E\left[ \BiasFifth_e F_e | \His_e \right]  \sim - \sum_e \frac{\gamma}{f_e(\pi_c) + \gamma}$ and that the random variable $|\BiasFifth_e F_e| \le 2L$ almost surely.
Thus, we can use  standard martingale concentration inequalities to get a tight enough concentration bound on $\sum_e \BiasFifth_e F_e - \E\left[ \BiasFifth_e F_e | \His_e \right]$ and further bound $\sum_e \BiasFifth_e F_e $. 
Finally, we have that $\sum_e \BiasFifth_e F_e = \sum_e \BiasFifth_e$ with high probability.
Thus, a high probability bound on $\sum_e \BiasFifth_e F_e $ transfers to a high probability bound on  $\sum_e \BiasFifth_e$.

%% file: conclusions.tex
\section{Conclusions}
We reanalyze the algorithm proposed by \citet{Zimmert2023} and show that it actually achieves near-optimal regret with \emph{high probability} for the cross-learning contextual bandits problem when the losses are chosen adversarially but the contexts are i.i.d.\ sampled from an \textit{unknown} distribution. Our key technique is utilizing the weak dependency structure between different epochs for an algorithm executing over multiple epochs. It is of interest to investigate  that whether this techniques is applicable for deriving high probability bounds for algorithms executing over multiple epochs in other problems.

%% file: appendix.tex
\section{Useful Lemmas}

\begin{lemma}[Freedman's Inequality]  
Fix any \(\lambda>0\) and \(\delta \in(0,1)\). Let \(X_t\) be a random process with respect to a filtration \(\mathcal{F}_t\) such that \(\mu_t = \mathbb{E}\left[X_t \mid \mathcal{F}_{t-1}\right]\) and \(V_t = \mathbb{E}\left[X_t^2 \mid \mathcal{F}_{t-1}\right]\), and satisfying \(\lambda X_t \leq 1\). Then, with probability at least \(1-\delta\), we have for all \(t\),  
$$  
\sum_{s=1}^t X_s - \mu_s \leq \lambda \sum_{s=1}^t V_s + \frac{\log(1 / \delta)}{\lambda}.  
$$  
\end{lemma}

The next lemma is about the following family of indicator functions.

\begin{definition}
For each epoch $e$, we define the following two indicator functions:
\[F_e \triangleq \indi\left(\forall a,\left|\widehat{f}_e(a)-f_e(a)\right| \leq 2 \max \left\{\sqrt{\frac{f_e(a) \iota}{L}}, \frac{\iota}{L}\right\}\right)\]
and 
\[L_e \triangleq \indi\left(\max _{c \in[C], a \in[K]} \sum_{t \in \mathcal{T}_e} \widetilde{\ell}_{t,c}(a) \leq L+\frac{\iota}{\gamma}\right).\]

We further define the following indicator function:
\[G=\prod_{e=1}^{T / L} F_e L_e.\]
\end{definition}

\begin{lemma}[Lemma 6 and Lemma 7, \citet{Zimmert2023}]
\label{lem:indicator_event}
For any epoch $e$, the event $F_e$ holds with probability at least $1 - 2K \exp(-\iota)$, and the event $L_e$ holds with probability at least $1- K \exp(-\iota)$. Furthermore, the event $G$ holds with probability at least $1 - 3K(T/L) \exp(-\iota).$
\end{lemma}

\begin{lemma}[Lemma 8, \citet{Zimmert2023} ]
\label{lem:estimator_constant_ratio}
Let $\gamma \geq \frac{4 \iota}{L}$, then under event $F_e$, we have that  $$ \frac{1}{2} \leq  \frac{f_e(a) + \gamma}{\widehat{f}_e(a)+\frac{3}{2} \gamma} \leq 2.$$
\end{lemma}

The next lemma is about the following auxiliary probability vector.
\begin{definition}
For each epoch $e$ and each time step $t \in \mathcal{T}_e$, we define 
\[\widetilde{p}_{t,c} \triangleq \argmin_{p \in \Delta([K])}\left\langle p, \sum_{e^{\prime}=1}^{e-1} \sum_{s \in \mathcal{T}_{e^{\prime}}} \widehat{\ell}_{s c}+\sum_{t^{\prime} \in \mathcal{T}_e, t^{\prime}<t} \widetilde{\ell}_{t^{\prime} c}\right\rangle - \eta^{-1} F(p)\]
where $F(p) = \sum_{i=1}^K p_i \log(p_i)$ is the unnormalized negative entropy.
\end{definition}

It is easy to see that $\widetilde{p}_{t,c} \propto s_{e+1, c} \circ \exp \left(-\eta \sum_{t^{\prime} \in \mathcal{T}_e, t^{\prime}<t} \widetilde{\ell}_{t^{\prime} c}\right).$

\begin{lemma}[Lemma 9, \citet{Zimmert2023} ]
\label{lem:probability_not_changed_much}
If $\gamma \geq \frac{4\iota}{L}$ and $\eta\leq  \frac{\log(2)}{5L}$, then under event $G$, we have for all $t\in \mathcal{T}_e, a \in[K],c\in[C]$ simultaneously
\begin{align*}
    2s_{e,c}(a) \geq p_{t,c}(a)\geq s_{e,c}(a)/2\qquad\text{and}\qquad
    2s_{e,c}(a)\geq \widetilde{p}_{t,c}(a)\geq s_{e,c}(a)/2\,.
\end{align*}
This implies that
\begin{align*}
    \E_{c \sim \nu}[p_{t,c}(a)] \leq 4 f_e(a)\qquad\text{and}\qquad
    \E_{c \sim \nu}[\widetilde{p}_{t,c}(a)] \leq 4 f_e(a)\,.
\end{align*}
In addition, this implies that $q_{t} = p_{t}$ for all $t \in \mathcal{T}_e$. 
\end{lemma}    

\begin{definition}
    We define $p_t(a) \triangleq \E_{c \sim \nu}[p_{t,c}(a)] $ and $\widetilde{p}_{t}(a) \triangleq \E_{c \sim \nu}[\widetilde{p}_{t,c}(a)] $ for each time step $t$ and each arm $a$.
\end{definition}

\begin{lemma}[Lemma 10, \citet{Zimmert2023} ]
\label{lem:expected_ratio}
If $\gamma \geq \frac{16 \iota}{L}$ and $ \exp (-\iota) \leq \frac{\gamma}{8 K}$, then
\[-\frac{\gamma}{f_{e}(a)} \leq \E\left[ \frac{f_e\left(a\right)-\widehat{f}_e\left(a\right)-\frac{1}{2} \gamma}{\widehat{f}_e\left(a\right)+\frac{3}{2} \gamma} F_e \Big| \His_{e-1} \right] \leq 0.\]
\end{lemma}
    
\begin{lemma}
\label{lem:about_product_concentration}
For any $\eta \leq \frac{\gamma}{2(2 L \gamma+\iota)}, \gamma \geq \frac{16 \iota}{L}, \iota \geq \log (8 K / \gamma)$, we have
\[\sum_{t \in \Tl} \sum_{c} \Pr(c) \left\langle p_{t, c}-\widetilde{p}_{t, c}, \widetilde{\ell}_{t, c}-\widehat{\ell}_{t, c}\right\rangle G \le\frac{98 K T \iota}{L}+\frac{\gamma^2 L K T}{\iota}.\]
\end{lemma}
\begin{proof}
The proof of \Cref{lem:about_product_concentration} is contained in the analysis of the $\bias_3$ term in \citet{Zimmert2023}.
\end{proof}

\begin{lemma}
\label{lemma:used_loss_reduction}
 Decomposing all time steps into consecutive pairs, specifically, decomposing $\{1,2,\dots,T\}$ into $\{ (1,2), (3,4), (5,6), \dots, (t-1,t), \dots, (T-1,T)\}$.
 Constructing a surrogate loss sequence $\{\widetilde{\ell}_s\}_{s=1}^{\frac{T}{2}}$ such that for each surrogate time step $s$ the loss vector $\widetilde{\ell}_s$ is uniformly sampled from the pair of true loss vector $(\ell_{2s-1},\ell_{2s})$.  Denote the time step sampled from the pair $(2s-1,2s)$ as $s_{\ell}$. 
 For any constant $\delta \in (0,1)$ and any bandit algorithm such that in each pair of time steps $(t-1,t)$, the algorithm takes actions $a_{t-1}$ and $a_t$ from the same distribution $p_{t-1} = p_t$, we have
\[ \sum_{t=1}^T \left( \ell_{t, c_t}(a_t) -  \ell_{t, c_t}(\pi_{c_t}) \right) \le 2 \sum_{s=1}^{\frac{T}{2}} \left( \widetilde{\ell}_{s,c_{s_{\ell}}}(a_{s_{\ell}}) - \widetilde{\ell}_{s,c_{s_{\ell}}}(\pi_{c_{s_{\ell}}}) \right) + 2 \sqrt{T \log(\frac{1}{\delta})} \]
with probability at least $1 - \delta$.
\end{lemma}


\begin{proof}[Proof of \cref{lemma:used_loss_reduction}]
    Consider the sequence of random variable $\{Y_s\}_{s=1}^{\frac{T}{2}}$ such that 
    \begin{align*}
        Y_s = & \ell_{2s-1, c_{2s-1}}(a_{2s-1}) -  \ell_{2s-1, c_{2s-1}}(\pi_{c_{2s-1}}) \\
        &+ \ell_{2s, c_{2s}}(a_{2s}) -  \ell_{2s, c_{2s}}(\pi_{c_{2s}})  \\
        &- 2 \left( \widetilde{\ell}_{s,c_{s_{\ell}}}(a_{s_{\ell}}) - \widetilde{\ell}_{s,c_{s_{\ell}}}(\pi_{c_{s_{\ell}}}) \right).
    \end{align*}
    Consider the filtration $\{\widetilde{H}_s\}_{s=1}^{\frac{T}{2}}$ such that for each $s$ the $\sigma$-field $\widetilde{H}_s$ is generated by the randomness within $c_t$ and $a_t$ for $t \le 2s$ and the randomness within sampling from pair $(2\tau-1,2\tau)$ for each $\tau \le s$. 
    It is easy to see that the sequence$\{Y_s\}_{s=1}^{\frac{T}{2}}$  forms a martingale difference sequence adapted to the filtration $\{\widetilde{H}_s\}_{s=1}^{\frac{T}{2}}$. Moreover, it is also easy to see that $|Y_s| \le 2$.  Using Azuma-Hoeffding's inequality, for any constant $\delta \in (0,1)$, we have
    \[\sum_{s=1}^{\frac{T}{2}} Y_s  \le 2\sqrt{T \log(\frac{1}{\delta})} \]
    with probability at least $1 - \delta$.
    This completes the proof of the lemma.
\end{proof}






\section[Detailed Proof of Our Theorem]{Detailed Proof of \Cref{thm:high_probability_main}}
\subsection{Decomposition}
As we mentioned in \Cref{sec:main_analysis}, we decompose the regret as 
\begin{align*}
    \Reg(\pi) &= \underbrace{\sum_{t=1}^T \ell_{t, c_t}(a_t) -  \ell_{t, c_t}(\pi_{c_t}) - 2 \sum_{t \in \Tl} \ell_{t, c_t} (a_t) -  \ell_{t, c_t}(\pi_{c_{t}})}_{\bias_1} \\
    +& \underbrace{2 \sum_{t \in \Tl}  \left( \ell_{t, c_{t}}(a_{t}) -  \ell_{t, c_{t}}(\pi_{c_{t}})  -   \sum_c \Pr(c)  \langle p_{t,c} - \pi_c , \ell_{t,c} \rangle \right)}_{\bias_2} \\
     +& \underbrace{2 \sum_{t \in \Tl}  \sum_{c} \Pr(c) \langle p_{t,c} - \pi_c, \widehat{\ell}_{t,c} \rangle}_{\textbf{ftrl}} + \underbrace{2 \sum_{t \in \Tl}  \sum_{c} \Pr(c) \left\langle p_{t,c}, \ell_{t,c} - \widetilde{\ell}_{t,c} 
     \right\rangle}_{\bias_3}\\
     +& \underbrace{2 \sum_{t \in \Tl}  \sum_{c} \Pr(c) \left\langle p_{t,c}, \widetilde{\ell}_{t,c} - \widehat{\ell}_{t,c} \right\rangle}_{\bias_4} + \underbrace{2 \sum_{t \in \Tl}  \sum_{c} \Pr(c) \left\langle \pi_c, \widehat{\ell}_{t,c} - \ell_{t,c} \right\rangle}_{\bias_5}.
\end{align*}
We bound these terms one by one. 

The $\bias_1$, $\bias_2$, $\textbf{ftrl}$, and $\bias_3$ terms are not hard to bound. The terms of interest to bound are $\bias_4$ ans $\bias_5$. We first bound these two terms. In these two terms, the $\bias_5$ term is the one easier to bound. We first bound $\bias_5$ to provide some intuition for our readers.

\subsection[Upper Bounding the Fifth Bias Term]{Upper Bounding $\bias_5$}

We first bound the fifth term $\sum_{t \in \Tl} \sum_{c} \Pr(c) \left\langle \pi_c, \widehat{\ell}_{t,c} - \ell_{t,c} \right\rangle$. 
%
We decompose the fifth term into two components:
\begin{align*}
&\sum_{t \in \Tl} \sum_{c} \Pr(c) \left\langle \pi_c, \widehat{\ell}_{t,c} - \ell_{t,c} \right\rangle\\
=& \sum_{t \in \Tl} \sum_{c} \Pr(c) \left( \widehat{\ell}_{t,c}(\pi_c) - \ell_{t,c}(\pi_c) \right)\\
=&  \sum_{t \in \Tl} \sum_{c} \Pr(c) \left( \E\left[ \widehat{\ell}_{t,c}(\pi_c) \big| \His_{t-1} \right]- \ell_{t,c}(\pi_c) \right) \\
&+ \sum_{t \in \Tl} \sum_{c} \Pr(c) \left( \widehat{\ell}_{t,c}(\pi_c) -  \E\left[ \widehat{\ell}_{t,c}(\pi_c) \big| \His_{t-1} \right] \right).
\end{align*}
We bound these two components separately.

%
%
For each epoch $e$, we define a random variable 
\[\BiasFifth_e \triangleq \sum_{t \in \Tel} \sum_{c} \Pr(c)  \left( \E\left[ \widehat{\ell}_{t,c}(\pi_c) \big| \His_{t-1} \right]- \ell_{t,c}(\pi_c) \right).\] 
We rewrite the term $\sum_{t \in \Tl} \sum_{c} \Pr(c) \left( \E\left[ \widehat{\ell}_{t,c}(\pi_c) \big| \His_{t-1} \right]- \ell_{t,c}(\pi_c) \right)$ as $\sum_{e=1}^{T/L} \BiasFifth_e$.
Recall our key observation: different epochs are only weakly dependent on each other. We bound the summation over epochs $\sum_{e=1}^{T/L} \BiasFifth_e$ by leveraging the weak dependence structure between $\left\{\BiasFifth_e\right\}_e$. 
%

The sequence of random variables $\{\BiasFifth_e\}_e$ has 
the following properties:
\begin{enumerate}
    \item For each epoch $e$, the random variable $ \BiasFifth_e$ is measurable under $\sigma$-algebra $\His_e$.
    %
    %
    \item For each epoch $e$, we have\footnote{Readers familiar with \citet{Zimmert2023} may wonder why we do not directly consider $\BiasFifth_e G$ but consider $\BiasFifth_e F_e$ instead. This is because there is a small flaw in the argument of \citet{Zimmert2023}. \citet{Zimmert2023} essentially argues that $\E[\BiasFifth_e G  | \His_{e-1}] = \E[\BiasFifth_e | \His_{e-1} ] \E[G | \His_{e-1}].$ However, this equality may not hold since the indicator $G$ depends on $\BiasFifth_e$ and these two terms are not conditionally independent given $\His_{e-1}$. This is why we consider
    $\BiasFifth_e F_e$ here instead.}
        \begin{align*}
            &\E\left[ \BiasFifth_e \cdot F_e \big| \His_{e-1} \right]\\
            =& \sum_c \Pr(c) \sum_{t \in \Tel} \ell_{t,c}(\pi_c) \E\left[ \left( \frac{f_e(\pi_c)}{\widehat{f}_e(\pi_c) + \frac{3}{2}\gamma} - 1 \right) F_e \big| \His_{e-1} \right].
        \end{align*}
    We further have 
    \begin{align*}
        &\E\left[ \left( \frac{f_e(\pi_c)}{\widehat{f}_e(\pi_c) + \frac{3}{2}\gamma} - 1 \right) F_e \big| \His_{e-1} \right]\\
        =& \E\left[ \left( \frac{f_e(\pi_c)}{\widehat{f}_e(\pi_c) + \frac{3}{2}\gamma} -  \frac{f_e(\pi_c)}{f_e(\pi_c) + \gamma} +\frac{f_e(\pi_c)}{f_e(\pi_c) + \gamma} - 1 \right) F_e \big| \His_{e-1} \right] \\
        =& \E\left[ \frac{f_e(\pi_c)}{f_e(\pi_c)+\gamma}  \frac{\left(f_e(\pi_c) - \widehat{f}_e(\pi_c) - \frac{1}{2} \gamma \right)}{\widehat{f}_e(\pi_c) + \frac{3}{2}\gamma} F_e \big| \His_{e-1}  \right] - \frac{\gamma}{f_e(\pi_c) + \gamma} \E\left[F_e \big| \His_{e-1}  \right] \\
        \le& - \frac{\gamma}{f_e(\pi_c) + \gamma} \E\left[F_e \big| \His_{e-1}  \right]\tag{\Cref{lem:expected_ratio}}\\
        \le&  - \frac{\gamma}{f_e(\pi_c) + \gamma} \left(  1-2 K \exp (-\iota) \right) \tag{\Cref{lem:indicator_event}}.
    \end{align*}
    Thus we have
    \begin{align*}
        &\E\left[ \BiasFifth_e \cdot F_e \big| \His_{e-1} \right]\\
        \le& - \sum_c \Pr(c) \sum_{t \in \Tel} \ell_{t,c}(\pi_c)\frac{\gamma}{f_e(\pi_c) + \gamma} \left(  1-2 K \exp (-\iota) \right).
    \end{align*}
  \item For each epoch $e$, we have
    \begin{align*}
        \BiasFifth_e F_e &\le  \sum_{c} \Pr(c) \sum_{t \in \Tel} \E\left[ \widehat{\ell}_{t,c}(\pi_c) \big| \His_{t-1} \right] F_e \\
        &= \sum_{c} \Pr(c) \sum_{t \in \Tel} \ell_{t,c}(\pi_c) \frac{f_e(\pi_c)}{\widehat{f}_e(\pi_c) + \frac{3}{2}\gamma} F_e \\
        &\le \sum_{c} \Pr(c) \sum_{t \in \Tel} \ell_{t,c}(\pi_c) \frac{2 f_e(\pi_c)}{f_e(\pi_c) + \gamma} \tag{\Cref{lem:estimator_constant_ratio}}\\
        &\le 2 L = \frac{32\iota}{\gamma}.
    \end{align*}
    \item For each epoch $e$, we have
    \begin{align*}
        &\E\left[ (\BiasFifth_e F_e)^2 \big| \His_{e-1}\right]\\ 
        =& \E\left[\left(\sum_{c} \Pr(c) \sum_{t \in \Tel}  \ell_{t,c}(\pi_c) \frac{f_e(\pi_c) - \widehat{f}_e(\pi_c) - \frac{3}{2}\gamma }{ \widehat{f}_e(\pi_c) + \frac{3}{2} \gamma} F_e \right)^2 \big| \His_{e-1} \right]\\
        &\le \sum_c \Pr(c) \E\left[\left(\sum_{t \in \Tel} \ell_{t,c}(\pi_c) \frac{f_e(\pi_c) - \widehat{f}_e(\pi_c) - \frac{3}{2}\gamma }{ \widehat{f}_e(\pi_c) + \frac{3}{2} \gamma} F_e \right)^2 \big| \His_{e-1} \right]\\
        &=  \sum_c \Pr(c) (\sum_{t \in \Tel} \ell_{t,c})^2 \E\left[ \left( \frac{f_e(\pi_c) - \widehat{f}_e(\pi_c) - \frac{3}{2}\gamma }{ \widehat{f}_e(\pi_c) + \frac{3}{2} \gamma} F_e \right)^2\big| \His_{e-1}\right] \\
        &\le L  \sum_c \Pr(c) \sum_{t \in \Tel} \ell_{t,c} \E\left[ \left( \frac{f_e(\pi_c) - \widehat{f}_e(\pi_c) - \frac{3}{2}\gamma }{ \widehat{f}_e(\pi_c) + \frac{3}{2} \gamma} F_e \right)^2\big| \His_{e-1}\right] \\
        %
        &\le 4 L \sum_c \Pr(c) \sum_{t \in \Tel} \ell_{t,c} \E\left[\left( \frac{f_e(\pi_c) - \widehat{f}_e(\pi_c) - \frac{3}{2}\gamma }{ f_e(\pi_c) + \gamma} F_e \right)^2\big| \His_{e-1} \right] \tag{\Cref{lem:estimator_constant_ratio}}\\
        &\le \sum_c \Pr(c)  \frac{4 L}{(f_e(\pi_c) + \gamma)^2} \sum_{t \in \Tel}  \ell_{t,c}(\pi_c) \E\left[ \left(f_e(\pi_c) - \widehat{f}_e(\pi_c) - \frac{3}{2}\gamma \right)^2 \big| \His_{e-1} \right] \\
        &= \sum_c \Pr(c) \frac{4 L}{(f_e(\pi_c) + \gamma)^2} \sum_{t \in \Tel}   \ell_{t,c}(\pi_c) \left(\E\left[ \left(f_e(\pi_c) - \widehat{f}_e(\pi_c)\right)^2 \big| \His_{e-1} \right] + \frac{9}{4}\gamma^2 \right) \\
        &\le \sum_c \Pr(c) \frac{4 L}{(f_e(\pi_c) + \gamma)^2} \sum_{t \in \Tel}  \sum_{t \in \Tel}  \ell_{t,c}(\pi_c) \left(\frac{f_e(\pi_c)}{L}+ \frac{9}{4}\gamma^2 \right) \\
        &\le 4 \sum_c \Pr(c) \sum_{t \in \Tel}  \ell_{t,c}(\pi_c) \left( \frac{1}{f_e(\pi_c) + \gamma} + \frac{9L \gamma }{4(f_e(\pi_c) + \frac{3}{2}\gamma)} \right) \\
        &\le 4 \sum_c \Pr(c) \sum_{t \in \Tel}  \ell_{t,c}(\pi_c) \frac{36\iota}{f_e(\pi_c) + \gamma}.
    \end{align*}
\end{enumerate}

Given these properties, we use Freedman's inequality to get that for any $0 < \lambda <\frac{ \gamma}{32\iota}$, with probability at least $1- \delta$, we have
\[\sum_e \BiasFifth_e F_e - \E[\BiasFifth_e F_e \big| \His_{e-1}]  \le 4 \lambda \sum_c \Pr(c) \sum_{t \in \Tel}  \ell_{t,c}(\pi_c) \frac{36\iota}{f_e(\pi_c) + \frac{3}{2} \gamma} + \frac{\log(1/\delta)}{\lambda}. \] 
We further have that event $\{ \sum_e \BiasFifth_e F_e = \sum_e \BiasFifth_e \}$ holds if event $G$ holds.
Combining  these two facts, we get that the inequality
\[\sum_e \BiasFifth_e  - \E[\BiasFifth_e F_e \big| \His_{e-1}]  \le 4 \lambda \sum_c \Pr(c) \sum_{t \in \Tel}  \ell_{t,c}(\pi_c) \frac{36\iota}{f_e(\pi_c) + \gamma} + \frac{\log(1/\delta)}{\lambda} \]
holds with probability at least $\Pr(G) - \delta$.

We now bound the second component \[\sum_{t \in \Tl} \sum_{c} \Pr(c) \left( \widehat{\ell}_{t,c}(\pi_c) -  \E\left[ \widehat{\ell}_{t,c}(\pi_c) \big| \His_{t-1} \right] \right).\]
The second term $\sum_{t \in \Tl} \sum_{c} \Pr(c) \left( \widehat{\ell}_{t,c}(\pi_c) -  \E\left[ \widehat{\ell}_{t,c}(\pi_c) \big| \His_{t-1} \right] \right)$ has the following properties:
\begin{enumerate}
    \item The sequence of random variables $\left\{ \sum_{c} \Pr(c) \left( \widehat{\ell}_{t,c}(\pi_c) -  \E\left[ \widehat{\ell}_{t,c}(\pi_c) \big| \His_{t-1} \right] \right)\right\}_{t \in \Tl}$ forms a martingale difference sequence with respect to the tfiltration $\{ \His_t \}_t$.
    \item Each random variable $\sum_{c} \Pr(c) \left( \widehat{\ell}_{t,c}(\pi_c) -  \E\left[ \widehat{\ell}_{t,c}(\pi_c) \big| \His_{t-1} \right] \right) $ satisfies $ \left| \sum_{c} \Pr(c) \left( \widehat{\ell}_{t,c}(\pi_c) -  \E\left[ \widehat{\ell}_{t,c}(\pi_c) \big| \His_{t-1} \right] \right) \right| \le \frac{1}{\gamma}$.
    \item Each random variable $\sum_{c} \Pr(c) \left( \widehat{\ell}_{t,c}(\pi_c) -  \E\left[ \widehat{\ell}_{t,c}(\pi_c) \big| \His_{t-1} \right] \right)$ satisfies
        \begin{align*}
        \Var \left[ \sum_{c} \Pr(c)\widehat{\ell}_{t,c}(\pi_c) \big| \His_{t-1} \right] &\le \sum_{c} \Pr(c) \Var\left[ \widehat{\ell}_{t,c}(\pi_c) \big| \His_{t-1} \right] \\
        &= \sum_{c} \Pr(c)  \frac{f_e(\pi_c) - f^2_e(\pi_c)}{\left(\widehat{f}_e(\pi_c)+ \frac{3}{2} \gamma \right)^2} l_{t, c}(\pi_c)^2  \\
        &\le  \sum_{c} \Pr(c)  \frac{f_e(\pi_c)}{\left(\widehat{f}_e(\pi_c)+ \frac{3}{2} \gamma \right)^2} l_{t, c}(\pi_c).
        \end{align*}
\end{enumerate}

Applying Freedman's inequality to the sequence of random variables $\left\{ \sum_{c} \Pr(c) \left( \widehat{\ell}_{t,c}(\pi_c) -  \E\left[ \widehat{\ell}_{t,c}(\pi_c) \big| \His_{t-1} \right] \right)\right\}_{t \in \Tl}$ , we get that for each $\delta \in (0,1)$ and each $0 < \lambda < \gamma$, with probability at least $1 - \delta$, we have
\begin{align*}
    &\sum_{t \in \Tl} \sum_{c} \Pr(c) \left( \widehat{\ell}_{t,c}(\pi_c) -  \E\left[ \widehat{\ell}_{t,c}(\pi_c) \big| \His_{t-1} \right] \right) \\
    \le& \lambda \sum_{t \in \Tl}  \sum_{c} \Pr(c)  \frac{f_e(\pi_c)}{\left(\widehat{f}_e(\pi_c)+ \frac{3}{2} \gamma \right)^2} l_{t, c}(\pi_c) + \frac{1}{\lambda} \log(\frac{1}{\delta}).
\end{align*}
By assuming that event $G$ holds, we further get  that with probability at least $\Pr(G) - \delta$, the following inequality holds:
\begin{align*}
    &\sum_{t \in \Tl} \sum_{c} \Pr(c) \left( \widehat{\ell}_{t,c}(\pi_c) -  \E\left[ \widehat{\ell}_{t,c}(\pi_c) \big| \His_{t-1} \right] \right) \\
    \le& 4 \lambda \sum_{t \in \Tl}  \sum_{c} \Pr(c)  \frac{f_e(\pi_c)}{\left(f_e(\pi_c)+ \gamma \right)^2} l_{t, c}(\pi_c) + \frac{1}{\lambda} \log(\frac{1}{\delta}) \tag{\Cref{lem:estimator_constant_ratio}}.
\end{align*}

Combining all previous inequalities, we get that for any $0 < \lambda_1 <\frac{ \gamma}{32\iota}$ and $0 < \lambda_2 < \gamma$ , the following inequality holds with probability at least $\Pr(G) - 2\delta $:
\begin{align*}
&\sum_{t \in \Tl} \sum_{c} \Pr(c) \left\langle \pi_c, \widehat{\ell}_{t,c} - \ell_{t,c} \right\rangle\\
=& \sum_{t \in \Tl} \sum_{c} \Pr(c) \left( \E\left[ \widehat{\ell}_{t,c}(\pi_c) \big| \His_{t-1} \right]- \ell_{t,c}(\pi_c) \right)  \\
&+ \sum_{t \in \Tl} \sum_{c} \Pr(c) \left( \widehat{\ell}_{t,c}(\pi_c) -  \E\left[ \widehat{\ell}_{t,c}(\pi_c) \big| \His_{t-1} \right] \right) \\
=& \sum_e \BiasFifth_e  - \E[\BiasFifth_e F_e \big| \His_{e-1}] \\
&+ \sum_e \E\left[\BiasFifth_e F_e \big| \His_{e-1}\right] \\
&+  \sum_{t \in \Tl} \sum_{c} \Pr(c) \left( \widehat{\ell}_{t,c}(\pi_c) -  \E\left[ \widehat{\ell}_{t,c}(\pi_c) \big| \His_{t-1} \right] \right)\\
\le& 4 \lambda_1 \sum_c \Pr(c) \sum_{t \in \Tl}  \ell_{t,c}(\pi_c) \frac{36\iota}{f_e(\pi_c) +  \gamma} + \frac{\log(1/\delta)}{\lambda_1} \\
&- \sum_c \Pr(c) \sum_{t \in \Tl}  \ell_{t,c}(\pi_c) \frac{\gamma}{f_e(\pi_c) + \gamma} \left(  1-2 K \exp (-\iota) \right) \\
&+ 4 \lambda_2 \sum_{t \in \Tl}  \sum_{c} \Pr(c)  \frac{f_e(\pi_c)}{\left(f_e(\pi_c) +  \gamma \right)^2} l_{t, c}(\pi_c) + \frac{1}{\lambda_2} \log(\frac{1}{\delta}).
\end{align*}

Note that in the previous analysis, we combined two good events each happening with probability at least $\Pr(G) - \delta$. The combined good event happens with probability $\Pr(G) - 2\delta$ rather than the vanilla union bound $1 - 2(1-\Pr(G) + \delta)$. This is because, in both events, the $\Pr(G)$ term comes from assuming event $G$ happens. Thus, in the combined event, we can simply assume event $G$ happens and count the corresponding bad event $G^c$ only once. We will use this small trick repeatedly in the following analysis.

We pick $\lambda_1 = \frac{\gamma}{8 \cdot 36 \iota}$ and $\lambda_2 = \frac{\gamma}{8} $  to get that the following inequality holds with probability at least $\Pr(G) - 2\delta$:
\begin{align*}
&\sum_{t \in \Tl} \sum_{c} \Pr(c) \left\langle \pi_c, \widehat{\ell}_{t,c} - \ell_{t,c} \right\rangle \\
\le& 4 \lambda_1 \sum_c \Pr(c) \sum_{t \in \Tl}  \ell_{t,c}(\pi_c) \frac{36\iota}{f_e(\pi_c) +  \gamma} + \frac{\log(1/\delta)}{\lambda_1} \\
&- \sum_c \Pr(c) \sum_{t \in \Tl}  \ell_{t,c}(\pi_c) \frac{\gamma}{f_e(\pi_c) + \gamma} \left(  1-2 K \exp (-\iota) \right) \\
&+ 4 \lambda_2 \sum_{t \in \Tl}  \sum_{c} \Pr(c)  \frac{f_e(\pi_c)}{\left(f_e(\pi_c) +  \gamma \right)^2} l_{t, c}(\pi_c) + \frac{1}{\lambda_2} \log(\frac{1}{\delta}) \\
=& (\frac{8 \cdot 36 \iota}{\gamma} + \frac{8}{\gamma}) \log(1/\delta) +  \sum_c \Pr(c) \sum_{t \in \Tl}  \ell_{t,c}(\pi_c)  \frac{\gamma}{f_e(\pi_c) + \gamma} 2K \exp(-\iota)\\
\le& (\frac{8 \cdot 36 \iota}{\gamma} + \frac{8}{\gamma}) \log(1/\delta) + KT \exp(-\iota).
\end{align*}

\subsection[Upper Bounding the Forth Bias Term]{Upper Bounding $\bias_4$}

We then bound the forth term $\sum_{t \in \Tl} \sum_{c} \Pr(c) \left\langle p_{t,c}, \widetilde{\ell}_{t,c} - \widehat{\ell}_{t,c} \right\rangle$. Similar to the previous analysis, we decompose it as follows:
\begin{align*}
    &\sum_{t \in \Tl} \sum_{c} \Pr(c) \left\langle p_{t,c}, \widetilde{\ell}_{t,c} - \widehat{\ell}_{t,c} \right\rangle\\
    =&\sum_{t \in \Tl} \sum_{c} \Pr(c) \left\langle p_{t,c}, \E\left[ \widetilde{\ell}_{t,c} - \widehat{\ell}_{t,c} \big| \His_{t-1} \right] \right\rangle\\
    &+\sum_{t \in \Tl} \sum_{c} \Pr(c) \left\langle p_{t,c}, \widetilde{\ell}_{t,c} - \widehat{\ell}_{t,c} - \E\left[ \widetilde{\ell}_{t,c} - \widehat{\ell}_{t,c} \big| \His_{t-1} \right] \right\rangle .
\end{align*}
We bound these two components separately.

Similar to the previous anlaysis, we decompose the first component \[\sum_{t \in \Tl} \sum_{c} \Pr(c) \left\langle p_{t,c}, \E\left[ \widetilde{\ell}_{t,c} - \widehat{\ell}_{t,c} \big| \His_{t-1} \right] \right\rangle\] as \[\sum_{e} \sum_{t \in \Tel} \sum_{c} \Pr(c) \left\langle p_{t,c}, \E\left[ \widetilde{\ell}_{t,c} - \widehat{\ell}_{t,c} \big| \His_{t-1} \right] \right\rangle.\]
For each epoch $e$ we define a random variable 
\[\BiasForth_e \triangleq \sum_{t \in \Tel} \sum_{c} \Pr(c) \left\langle p_{t,c}, \E\left[ \widetilde{\ell}_{t,c} - \widehat{\ell}_{t,c} \big| \His_{t-1} \right] \right\rangle.\]
We need to bound $\sum_e \BiasForth_e$.

We decompose $\sum_e \BiasForth_e$ as $\sum_e \BiasForth F_eL_e + \sum_e \BiasForth_e (1- F_eL_e)$. 
As usual we have that $\sum_e \BiasForth_e (1- F_eL_e) = 0$ whenever event $G$ holds. 
Thus we can focus on bounding $\sum_e \BiasForth_e F_eL_e$. 
Firstly we bound \[\sum_e \E\left[ \BiasForth_e F_e L_e \big| \His_{e-1} \right].\]

We have 
\begin{align*}
    &\sum_e \E\left[ \BiasForth_e F_e L_e \big| \His_{e-1} \right] \\
    =&  \sum_{e} \E\left[ \sum_{t \in \Tel} \sum_{c} \Pr(c) \left\langle \widetilde{p}_{t,c}, \E\left[ \widetilde{\ell}_{t,c} - \widehat{\ell}_{t,c} \big| \His_{t-1} \right] \right\rangle F_e L_e \big| \His_{e-1} \right]\\
    &+ \sum_{e} \E\left[ \sum_{t \in \Tel} \sum_{c} \Pr(c) \left\langle p_{t,c} - \widetilde{p}_{t,c}, \E\left[ \widetilde{\ell}_{t,c} - \widehat{\ell}_{t,c} \big| \His_{t-1} \right] \right\rangle F_e L_e \big| \His_{e-1} \right].
\end{align*}

By \Cref{lem:about_product_concentration}, the latter term \[\sum_{e} \E\left[ \sum_{t \in \Tel} \sum_{c} \Pr(c) \left\langle p_{t,c} - \widetilde{p}_{t,c}, \E\left[ \widetilde{\ell}_{t,c} - \widehat{\ell}_{t,c} \big| \His_{t-1} \right] \right\rangle F_e L_e \big| \His_{e-1} \right]\] is bounded by  $\frac{98 K T \iota}{L}+\frac{\gamma^2 L K T}{\iota}.$ 
Furthermore, condition on $\His_{e-1}$, the indicator function $F_e$ is effected only by randomness within time steps $t \in \Tef$, thus the indicator function $F_e$ is conditional independent with the probability vector $\widetilde{p}_{t,c}$. 
We have 
\begin{align*}
    &\sum_e \E\left[ \sum_{t \in \Tel} \sum_{c} \Pr(c) \left\langle \widetilde{p}_{t,c}, \E\left[ \widetilde{\ell}_{t,c} - \widehat{\ell}_{t,c} \big| \His_{t-1} \right] \right\rangle F_e \big| \His_{e-1} \right]\\
    =& \sum_e \sum_{t \in \Tel} \sum_{c} \Pr(c) \left\langle \E\left[ \widetilde{p}_{t,c} \big| \His_{e-1} \right], \E\left[ (\widetilde{\ell}_{t,c} - \widehat{\ell}_{t,c}) F_e \big| \His_{e-1} \right] \right\rangle \\
    \le& \sum_e \sum_{t \in \Tel} \sum_{c} \Pr(c) \sum_a \frac{\widetilde{p}_{t, c}(a) \gamma}{f_{e}(a)} \tag{\Cref{lem:expected_ratio}}.
\end{align*}

By \Cref{lem:probability_not_changed_much}, whenever event $G$ holds, the ratio $\frac{\widetilde{p}_{t, c}(a)}{f_{e}(a)} \le 4.$ Thus we have \[\sum_e \E\left[ \sum_{t \in \Tel} \sum_{c} \Pr(c) \left\langle \widetilde{p}_{t,c}, \E\left[ \widetilde{\ell}_{t,c} - \widehat{\ell}_{t,c} \big| \His_{t-1} \right] \right\rangle F_e \big| \His_{e-1} \right] \le 4 \gamma K  T\] whenever event $G$ holds.

We further have $\left| \left\langle \widetilde{p}_{t,c}, \E\left[ \widetilde{\ell}_{t,c} - \widehat{\ell}_{t,c} \big| \His_{t-1} \right] \right\rangle \right| \le \frac{1}{\gamma}$.
Thus we have
\begin{align*}
    & \sum_e \E\left[ \sum_{t \in \Tel} \sum_{c} \Pr(c) \left\langle \widetilde{p}_{t,c}, \E\left[ \widetilde{\ell}_{t,c} - \widehat{\ell}_{t,c} \big| \His_{t-1} \right] \right\rangle F_e (L_e-1) \big| \His_{e-1} \right]  \\
    \le& \sum_e \E\left[ \sum_{t \in \Tel} \sum_{c} \Pr(c) \frac{1}{\gamma} F_e (L_e-1) \big| \His_{e-1} \right]  \\
    \le& \frac{1}{\gamma} \sum_e \E\left[ \sum_{t \in \Tel} \sum_{c} \Pr(c)  \left| L_e - 1 \right| \big| \His_{e-1} \right]  \\
    %
    %
    \le& \frac{K \exp(-\iota) T}{\gamma} \tag{\Cref{lem:indicator_event}}.
\end{align*}

Thus we have 
\begin{align*}
   &\sum_e \E\left[ \sum_{t \in \Tel} \sum_{c} \Pr(c) \left\langle \widetilde{p}_{t,c}, \E\left[ \widetilde{\ell}_{t,c} - \widehat{\ell}_{t,c} \big| \His_{t-1} \right] \right\rangle F_eL_e \big| \His_{e-1} \right]\\
   \le& 4 \gamma K  T +  \frac{K \exp(-\iota) T}{\gamma} 
\end{align*}
whenever event $G$ holds.

Thus we have
\[\sum_e \E\left[ \BiasForth_e F_e L_e \big| \His_{e-1} \right] \le4 K \gamma T +  \frac{K \exp(-\iota) T}{\gamma} + \frac{98 K T \iota}{L}+\frac{\gamma^2 L K T}{\iota} \]
whenever event $G$ holds.

We then only need to bound the concentration term \[\sum_e \BiasForth_e F_e L_e- \E\left[ \BiasForth_e F_e L_e \big| \His_{e-1} \right].\]

For each random variable $\BiasForth_e F_e L_e $, we have

\begin{align*}
    &\BiasForth_e F_e L_e\\
    =& \sum_{t \in \Tel} \sum_{c} \Pr(c) \left\langle p_{t,c}, \E\left[ \widetilde{\ell}_{t,c} - \widehat{\ell}_{t,c} \big| \His_{t-1} \right] \right\rangle F_e L_e\\
     =& \sum_{t \in \Tel} \sum_{c} \Pr(c) \sum_a p_{t,c}(a) f_e(a) \ell_{t,c}(a) \left( \frac{1}{f_e(a) + \gamma} - \frac{1}{\widehat{f}_e(a) + \frac{3}{2} \gamma} \right) F_e L_e \\
    =& \sum_{t \in \Tel} \sum_{c} \Pr(c) \sum_a \widetilde{p}_{t,c}(a) f_e(a) \ell_{t,c}(a) \frac{\widehat{f}_e(a) - f_e(a) + \frac{1}{2}\gamma}{(f_e(a) + \gamma)(\widehat{f}_e(a) + \frac{3}{2} \gamma)} F_e L_e.\\
    %
\end{align*}

Thus we have
\begin{align*}
    &\left| \BiasForth_e F_eL_e \right|\\
    \le& \left| \sum_{t \in \Tel} \sum_{c} \Pr(c) \sum_a \widetilde{p}_{t,c}(a) f_e(a) \ell_{t,c}(a)  \frac{\widehat{f}_e(a) - f_e(a) + \frac{1}{2}\gamma}{(f_e(a) + \gamma)(\widehat{f}_e(a) + \frac{3}{2} \gamma)} \right| F_eL_e \\
    \le& \sum_{t \in \Tel} \sum_a \widetilde{p}_{t}(a) \left| \frac{\widehat{f}_e(a) - f_e(a) + \frac{1}{2}\gamma}{\widehat{f}_e(a) + \frac{3}{2} \gamma}  \right| F_eL_e \\
    %
    %
    \le& 8  \sum_{t \in \Tel} \sum_a \max \left\{\sqrt{\frac{f_{e}(a) \iota}{L}}, \frac{\iota}{L} \right\} \\
    \le& 8 L (\sqrt{\frac{K\iota}{L}} + \frac{K\iota}{L})\\
    =&   8(\sqrt{KL\iota} + K\iota).
\end{align*}

Applying Azuma-Hoeffding's inequality to \[\sum_e  \BiasForth_e F_eL_e -  \E\left[\BiasForth_e F_eL_e \big| \His_{e-1} \right],\] 
we get that for any $\delta > 0$,  with probability at least $1 - \delta$, we have
\begin{align*}
    &\sum_e  \BiasForth_e F_eL_e -  \E\left[\BiasForth_e F_eL_e \big| \His_{e-1} \right]\\
    \le&  8(\sqrt{KL\iota} + K\iota)\sqrt{2\frac{T}{L}\log(\frac{\delta}{2})} \\
    =& 8\left(\sqrt{2KT\iota \log(\frac{\delta}{2})} + \sqrt{2\frac{TK^2}{L}\iota \log(\frac{\delta}{2})}\right).
\end{align*}

Thus we have 
\begin{align*}
    &\sum_e  \BiasForth_e F_eL_e \\
    \le& 4 K \gamma T +  \frac{K \exp(-\iota) T}{\gamma} + \frac{98 K T \iota}{L}+\frac{\gamma^2 L K T}{\iota} \\
    &+ 8\left(\sqrt{2KT\iota \log(\frac{\delta}{2})} + \sqrt{2\frac{TK^2}{L}\iota \log(\frac{\delta}{2})}\right).
\end{align*}
with probability at least $\Pr(G) - \delta$.

We then bound the second term \[\sum_{t \in \Tl} \sum_{c} \Pr(c) \left\langle p_{t,c}, \widetilde{\ell}_{t,c} - \widehat{\ell}_{t,c} - \E\left[ \widetilde{\ell}_{t,c} - \widehat{\ell}_{t,c} \big| \His_{t-1} \right] \right\rangle.\] 
For each time step $t \in \Tel$, we define an indicator function \[J_t \triangleq \indi\left(\forall a, \widetilde{p}_{t}(a) \le 4 f_e(a) \right).\] 
By \Cref{lem:probability_not_changed_much}, event $G$ implies $J_t$.

Similar to previous analysis, we decompose the first term as
\begin{align*}
    &\sum_{t \in \Tl} \sum_{c} \Pr(c) \left\langle p_{t,c}, \widetilde{\ell}_{t,c} - \widehat{\ell}_{t,c} - \E\left[ \widetilde{\ell}_{t,c} - \widehat{\ell}_{t,c} \big| \His_{t-1} \right] \right\rangle \\
    =&\sum_{t \in \Tl} \sum_{c} \Pr(c) \left\langle p_{t,c}, \widetilde{\ell}_{t,c} - \widehat{\ell}_{t,c} - \E\left[ \widetilde{\ell}_{t,c} - \widehat{\ell}_{t,c} \big| \His_{t-1} \right] \right\rangle F_eJ_t \\
    &+ \sum_{t \in \Tl} \sum_{c} \Pr(c) \left\langle p_{t,c}, \widetilde{\ell}_{t,c} - \widehat{\ell}_{t,c} - \E\left[ \widetilde{\ell}_{t,c} - \widehat{\ell}_{t,c} \big| \His_{t-1} \right] \right\rangle (1-F_e J_t).
\end{align*}

Since the auxiliary probability vector $\widetilde{p}_{t,c}$ is determined at time $t-1$, the indicator function $J_t$ is also determined at time $t-1$. 
Furthermore, the indicator function $F_e$ is determined at epoch $e-1$.
Thus the product of indicator functions $F_e J_t$ is measurable under filtration $\His_{t-1}$. We have
\begin{align*}
    &\E\left[\sum_{c} \Pr(c) \left\langle p_{t,c}, \widetilde{\ell}_{t,c} - \widehat{\ell}_{t,c} - \E\left[ \widetilde{\ell}_{t,c} - \widehat{\ell}_{t,c} \big| \His_{t-1} \right] \right\rangle F_e J_t\big| \His_{t-1} \right]\\
    =& F_eJ_t \E\left[  \sum_{c} \Pr(c) \left\langle p_{t,c}, (\widetilde{\ell}_{t,c} - \widehat{\ell}_{t,c})  - \E\left[ (\widetilde{\ell}_{t,c} - \widehat{\ell}_{t,c}) \big| \His_{t-1} \right] \right\rangle \big| \His_{t-1} \right] \\
    =& 0.
\end{align*}
Thus the sequence of random variables \[\left\{ \sum_{c} \Pr(c) \left\langle p_{t,c}, \widetilde{\ell}_{t,c} - \widehat{\ell}_{t,c} - \E\left[ \widetilde{\ell}_{t,c} - \widehat{\ell}_{t,c} \big| \His_{t-1} \right] \right\rangle F_e J_t \right\}_{t \in \Tl}\] forms a martingale difference sequence under the filtration $\{\His_t\}_t$.

We further have that the term 
\[\sum_{c} \Pr(c) \left\langle p_{t,c}, \widetilde{\ell}_{t,c} - \widehat{\ell}_{t,c} - \E\left[ \widetilde{\ell}_{t,c} - \widehat{\ell}_{t,c} \big| \His_{t-1} \right] \right\rangle F_eJ_t \] satisfies
    \begin{align*}
        &\left| \sum_{c} \Pr(c) \left\langle p_{t,c}, \widetilde{\ell}_{t,c} - \widehat{\ell}_{t,c} - \E\left[ \widetilde{\ell}_{t,c} - \widehat{\ell}_{t,c} \big| \His_{t-1} \right] \right\rangle \right| F_eJ_t \\
        \le& \left| \sum_{c} \Pr(c) \sum_{a} p_{t,c}(a) \ell_{t,c}(a) \left(  \frac{\indi(a_t = a)}{f_e(a) + \gamma} -  \frac{f_e(a)}{f_e(a) + \gamma}\right) \right| F_eJ_t \\
        &+ \left|\sum_{c} \Pr(c) \sum_{a} p_{t,c}(a) \ell_{t,c}(a) \left(  \frac{\indi(a_t = a)}{\widehat{f}_e(a) + \frac{3}{2}\gamma} -  \frac{f_e(a)}{\widehat{f}_e(a) + \frac{3}{2}\gamma}\right) \right| F_eJ_t \\
        \le& \sum_{c} \Pr(c) p_{t,c}(a_t) \frac{\ell_{t,c}(a_t) }{f_e(a_t) + \gamma} F_eJ_t \\
        &+ \sum_{c} \Pr(c) \sum_{a} p_{t,c}(a) \ell_{t,c}(a) \frac{f_e(a)}{f_e(a) + \gamma} F_eJ_t \\
        &+ \sum_{c} \Pr(c) p_{t,c}(a_t) \frac{\ell_{t,c}(a_t) }{\widehat{f}_e(a_t) + \frac{3}{2} \gamma} F_e J_t\\
        &+ \sum_{c} \Pr(c) \sum_{a} p_{t,c}(a) \ell_{t,c}(a) \frac{f_e(a)}{\widehat{f}_e(a) + \frac{3}{2}\gamma} F_eJ_t \\
        \le& \left( \frac{p_t(a_t)}{f_e(a_t) + \gamma} + 1 +  \frac{p_t(a_t)}{\widehat{f}_e(a_t) + \frac{3}{2}\gamma} + \sum_{a} p_t(a) \frac{f_e(a)}{\widehat{f}_e(a) + \frac{3}{2}\gamma} \right) F_e J_t\\
        \le& 4+1+8+2 = 15. \tag{\Cref{lem:estimator_constant_ratio}}
    \end{align*}

Applying Azuma-Hoeffding's inequality to the sequnece of random variables
\[\left\{ \sum_{c} \Pr(c) \left\langle p_{t,c}, \widetilde{\ell}_{t,c} - \widehat{\ell}_{t,c} - \E\left[ \widetilde{\ell}_{t,c} - \widehat{\ell}_{t,c} \big| \His_{t-1} \right] \right\rangle F_eJ_t \right\}_{t \in \Tl},\]
we get that for any $\delta > 0$,
the inequality
\[   \sum_{t \in \Tl} \sum_{c} \Pr(c) \left\langle p_{t,c}, \widetilde{\ell}_{t,c} - \widehat{\ell}_{t,c} - \E\left[ \widetilde{\ell}_{t,c} - \widehat{\ell}_{t,c} \big| \His_{t-1} \right] \right\rangle F_e J_t \le 15 \sqrt{T\log(\frac{1}{\delta})}\] 
holds with probability at least $1- \delta$.

On the other hand, note that event $G$ implies event $F_eJ_t$, we get that 
\[\sum_{t \in \Tl} \sum_{c} \Pr(c) \left\langle p_{t,c}, \widetilde{\ell}_{t,c} - \widehat{\ell}_{t,c} - \E\left[ \widetilde{\ell}_{t,c} - \widehat{\ell}_{t,c} \big| \His_{t-1} \right] \right\rangle (1-F_eJ_t) = 0 \]
whenever event $G$ holds. Thus we get that for any $\delta > 0$, inequality
\[  \sum_{t \in \Tl} \sum_{c} \Pr(c) \left\langle p_{t,c}, \widetilde{\ell}_{t,c} - \widehat{\ell}_{t,c} - \E\left[ \widetilde{\ell}_{t,c} - \widehat{\ell}_{t,c} \big| \His_{t-1} \right] \right\rangle \le 15 \sqrt{T\log(\frac{1}{\delta})}\] 
holds with probability at least $\Pr(G) - 2\delta$.

Combining previous results, we get that the following inequality holds with probability at least $\Pr(G) - 2\delta$:
\begin{align*}
    &\sum_{t \in \Tl} \sum_{c} \Pr(c) \left\langle p_{t,c}, \widetilde{\ell}_{t,c} - \widehat{\ell}_{t,c} \right\rangle\\
    =&\sum_{t \in \Tl} \sum_{c} \Pr(c) \left\langle p_{t,c}, \E\left[ \widetilde{\ell}_{t,c} - \widehat{\ell}_{t,c} \big| \His_{t-1} \right] \right\rangle\\
    &+\sum_{t \in \Tl} \sum_{c} \Pr(c) \left\langle p_{t,c}, \widetilde{\ell}_{t,c} - \widehat{\ell}_{t,c} - \E\left[ \widetilde{\ell}_{t,c} - \widehat{\ell}_{t,c} \big| \His_{t-1} \right] \right\rangle \\
    \le& 4 K \gamma T +  \frac{K \exp(-\iota) T}{\gamma} + \frac{98 K T \iota}{L}+\frac{\gamma^2 L K T}{\iota} \\
    &+ 8\left(\sqrt{2KT\iota \log(\frac{\delta}{2})} + \sqrt{2\frac{TK^2}{L}\iota \log(\frac{\delta}{2})}\right)\\
    &+  15 \sqrt{T\log(\frac{1}{\delta})}.
\end{align*}

\subsection{Upper Bounding Remaining Terms}

The remaining terms are the $\bias_1$, $\bias_2$, $\textbf{ftrl}$, and $\bias_3$ terms. These terms are not hard to bound using techniques in standard EXP3-IX analysis~\citep{EXP3-IX}. We write down these analyses in the sake of completeness. 

\subsubsection[Upper Bounding the First Bias Term]{Upper Bounding $\bias_1$}
Applying \Cref{lemma:used_loss_reduction} on the loss sequence used in calculating loss estimates, we get that
\[\sum_{t=1}^T \ell_{t, c_t}(a_t) -  \ell_{t, c_t}(\pi_{c_t}) - 2 \sum_{t \in \Tl} \ell_{t, c_t} (a_t) -  \ell_{t, c_t}(\pi_{c_{t}}) \le  2 \sqrt{T \log(\frac{1}{\delta})}. \]

\subsubsection[Upper Bounding the Second Bias Term]{Upper Bounding $\bias_2$}
Here we bound the $\bias_2$ term  \[\sum_{t \in \Tl}  \left( \ell_{t, c_{t}}(a_{t}) -  \ell_{t, c_{t}}(\pi_{c_{t}})  -   \sum_c \Pr(c)  \langle p_{t,c} - \pi_c , \ell_{t,c} \rangle \right). \] 

Whenever event $G$ holds, we have $q_t = p_t$. Thus we assume event $G$ holds and replace the $\bias_2$ term by \[\sum_{t \in \Tl}  \left( \ell_{t, c_{t}}(a_{t}) -  \ell_{t, c_{t}}(\pi_{c_{t}})  -   \sum_c \Pr(c)  \langle q_{t,c} - \pi_c , \ell_{t,c} \rangle \right). \] 

The new term have the following properties:
\begin{itemize}
    \item The sequence of random variables \[\left\{ \ell_{t, c_{t}}(a_{t}) -  \ell_{t, c_{t}}(\pi_{c_{t}})  -   \sum_c \Pr(c)  \langle q_{t,c} - \pi_c , \ell_{t,c} \rangle \right\}_{t \in \Tl}\] adapts to the filtration $\left\{\His_t\right\}_{t \in \Tl}$.
    \item The sequence of random variables satisfies \[\E\left[\ell_{t, c_{t}}(a_{t}) -  \ell_{t, c_{t}}(\pi_{c_{t}})  -   \sum_c \Pr(c)  \langle q_{t,c} - \pi_c , \ell_{t,c} \rangle \big| \His_{t-1} \right] = 0.\]
    \item Each random variable satisfies $ \ell_{t, c_{t}}(a_{t}) -  \ell_{t, c_{t}}(\pi_{c_{t}})  -   \sum_c \Pr(c)  \langle q_{t,c} - \pi_c , \ell_{t,c} \rangle\in [-2,2]$. 
\end{itemize}

By applying Azuma-Hoeffding inequality, we get that for any $\delta \in (0,1)$, the following inequality holds with probability at least $1 - \delta$:
\begin{align*}
    &\sum_{t \in \Tl}  \left( \ell_{t, c_{t}}(a_{t}) -  \ell_{t, c_{t}}(\pi_{c_{t}})  -   \sum_c \Pr(c)  \langle q_{t,c} - \pi_c , \ell_{t,c} \rangle \right)\\
    \le& 2 \sqrt{T \log(\frac{1}{\delta})}
\end{align*}
Thus the following inequality holds with probabiity at least $\Pr(G) - \delta$:
\[\sum_{t \in \Tl}  \left( \ell_{t, c_{t}}(a_{t}) -  \ell_{t, c_{t}}(\pi_{c_{t}})  -   \sum_c \Pr(c)  \langle p_{t,c} - \pi_c , \ell_{t,c} \rangle \right) \le 2 \sqrt{T \log(\frac{1}{\delta})}.\]

\subsubsection[Upper Bounding the ftrl Term]{Upper Bounding $\textbf{ftrl}$}
Here we bound the $\textbf{ftrl}$ term  \[ \sum_{t \in \Tl}  \sum_{c} \Pr(c) \langle p_{t,c} - \pi_c, \widehat{\ell}_{t,c} \rangle. \]

By the standard analysis of FTRL algorithms, the $\textbf{ftrl}$ term satisfies
\begin{align*}
    &\sum_{t \in \Tl}  \sum_{c} \Pr(c) \langle p_{t,c} - \pi_c, \widehat{\ell}_{t,c} \rangle\\
    \le&\sum_{c} \Pr(c) \left( \frac{1}{\eta} \log K  + \frac{\eta}{2} \sum_{t \in \Tl} \left\langle p_{t,c}, \widehat{\ell}^2_{t,c} \right\rangle \right).
\end{align*}
Here $\widehat{\ell}_{t,c}^2$ denotes the vector formed by squaring each component of $\widehat{\ell}_{t,c}$.

By \Cref{lem:estimator_constant_ratio}, under event $G$, we have $\widehat{\ell}_{t,c} \le 2 \widetilde{\ell}_{t,c}$. Thus assuming event $G$ holds, we can focus on upper bounding \[\sum_{c} \Pr(c) \left( \frac{1}{\eta} \log K  + 2 \eta \sum_{t \in \Tl} \left\langle p_{t,c}, \widetilde{\ell}^2_{t,c} \right\rangle \right).\] 

It suffices to upper bound $\sum_c \Pr(c) \sum_{t \in \Tl} \left\langle p_{t,c}, \widetilde{\ell}^2_{t,c} \right\rangle .$ We have
\begin{align*}
    &\sum_c \Pr(c) \sum_{t \in \Tl} \left\langle p_{t,c}, \widetilde{\ell}^2_{t,c} \right\rangle \\
    =&\sum_c \Pr(c) \sum_{t \in \Tl} p_{t,c}(a_t) \widetilde{\ell}^2_{t,c}(a_t) \\
    =&\sum_c \Pr(c) \sum_{t \in \Tl} p_{t,c}(a_t) \frac{\ell^2_{t,c}(a_t)}{\left(f_e(a_t) + \gamma\right)^2} \\
    \le& \sum_{t \in \Tl} \sum_{c} \Pr(c) p_{t,c}(a_t) \frac{1}{\left(f_e(a_t) + \gamma\right)^2} \\
    =&  \sum_{t \in \Tl} \frac{p_t(a_t)}{\left(f_e(a_t) + \gamma\right)^2}. 
\end{align*}
By \Cref{lem:probability_not_changed_much}, under event $G$, we have \[\sum_{t \in \Tl} \frac{p_t(a_t)}{\left(f_e(a_t) + \gamma\right)^2} \le 2 \sum_{t \in \Tl} \frac{1}{f_e(a_t) + \gamma}. \] 
We then focus on upper bounding $ \sum_{t \in \Tl} \frac{1}{f_e(a_t) + \gamma}.$ 

We have
\begin{itemize}
    \item The sum of conditional expectations $\sum_{t \in \Tl} \E\left[\frac{1}{f_e(a_t) + \gamma} \big| \His_{t-1} \right] \le KT.$
    \item Each term $\frac{1}{f_e(a_t) + \gamma} \le \frac{1}{\gamma}.$
    \item The sum of conditional quadratic expectations
        \begin{align*}
            &\sum_{t \in \Tl} \E\left[ \left(\frac{1}{f_e(a_t) + \gamma}\right)^2 \big| \His_{t-1} \right]\\
            =& \sum_{t \in \Tl} \sum_a \frac{f_e(a)}{\left(f_e(a) + \gamma\right)^2}\\
            \le& \sum_{t \in \Tl} \sum_a \frac{1}{f_e(a) + \gamma}.
        \end{align*}
\end{itemize}

By Freedman's inequality, we have that for any $\lambda \in (0,\gamma]$ and any $\delta \in (0,1)$, the following inequality holds with probability at least $1-\delta$:
\[ \sum_{t \in \Tl} \frac{1}{f_e(a_t) + \gamma} \le \lambda \sum_{t \in \Tl} \sum_a \frac{1}{f_e(a) + \gamma} + \frac{1}{\lambda} \log(\frac{1}{\delta}) + KT. \]
We pick $\lambda = \gamma$ to get that
\[ \sum_{t \in \Tl} \frac{1}{f_e(a_t) + \gamma} \le 2 KT + \frac{1}{\gamma} \log(\frac{1}{\delta}) \]
with probability at least $1-\delta$.

Substituting this inequality back, we get that the folowing inequality holds with probability at least $\Pr(G) - \delta$:
\begin{align*}
    &\sum_{c} \Pr(c) \left( \frac{1}{\eta} \log K  + \frac{\eta}{2} \sum_{t \in \Tl} \left\langle p_{t,c}, \widehat{\ell}^2_{t,c} \right\rangle \right) \\
    \le& \frac{1}{\eta} \log K + 4 \eta  \sum_{t \in \Tl} \frac{1}{f_e(a_t) + \gamma}\\
    \le&  \frac{1}{\eta} \log K + 4 \eta \left( 2 KT + \frac{1}{\gamma} \log(\frac{1}{\delta}) \right).
\end{align*}

\subsubsection[Upper Bounding the Third Bias Term]{Upper Bounding $\bias_3$}
Here we bound the $\bias_3$ term \[ \sum_{t \in \Tl}  \sum_{c} \Pr(c) \left\langle p_{t,c}, \ell_{t,c} - \widetilde{\ell}_{t,c} \right\rangle. \] 
We decompose it as
\begin{align*}
    &\sum_{t \in \Tl}  \sum_{c} \Pr(c) \left\langle p_{t,c}, \ell_{t,c} - \widetilde{\ell}_{t,c} \right\rangle\\
    =& \sum_{t \in \Tl} \sum_c \Pr(c) \left\langle p_{t,c}, \ell_{t,c} - \E\left[\widetilde{\ell}_{t,c}\big| \His_{t-1}\right] \right\rangle \\
    &+ \sum_{t \in \Tl} \sum_c \Pr(c) \left\langle p_{t,c}, \E\left[\widetilde{\ell}_{t,c}\big| \His_{t-1}\right] - \widetilde{\ell}_{t,c} \right\rangle.
\end{align*}

The first component satisfies 
\begin{align*}
    &\sum_{t \in \Tl} \sum_c \Pr(c) \left\langle p_{t,c}, \ell_{t,c} - \E\left[\widetilde{\ell}_{t,c}\big| \His_{t-1}\right] \right\rangle \\
    =&\sum_{t \in \Tl} \sum_c \Pr(c) \sum_{a} p_{t,c}(a) \ell_{t,c}(a)  \frac{\gamma}{f_e(a) + \gamma} \\
    \le& \sum_{t \in \Tl} \sum_{a} p_{t}(a)  \frac{\gamma}{f_e(a) + \gamma}.
\end{align*}
Assuming event $G$ holds, we have
\begin{align*}
    &\sum_{t \in \Tl} \sum_{a} p_{t}(a)  \frac{\gamma}{f_e(a) + \gamma} \\
    \le& 2 \sum_{t \in \Tl} \sum_{a} \gamma \le 2KT\gamma. \tag{\Cref{lem:probability_not_changed_much}}
\end{align*}

We then bound the second component $\sum_{t \in \Tl} \sum_c \Pr(c) \left\langle p_{t,c}, \E\left[\widetilde{\ell}_{t,c}\big| \His_{t-1}\right] - \widetilde{\ell}_{t,c} \right\rangle.$ 
For each time step $t \in \Tl$, we define an indicator function \[L_t \triangleq \indi\left(\forall a, p_{t}(a) \le 4 f_e(a) \right).\] 
By \Cref{lem:probability_not_changed_much},  event $G$ implies event $L_t$.
Thus we have 
\begin{align*}
    &\sum_{t \in \Tl} \sum_c \Pr(c) \left\langle p_{t,c}, \E\left[\widetilde{\ell}_{t,c}\big| \His_{t-1}\right] - \widetilde{\ell}_{t,c} \right\rangle \\
    =& \sum_{t \in \Tl} \sum_c \Pr(c) \left\langle p_{t,c}, \E\left[\widetilde{\ell}_{t,c}\big| \His_{t-1}\right] - \widetilde{\ell}_{t,c} \right\rangle L_t.
\end{align*}
under event $G$. We then assume event $G$ holds and focus on upper bounding \[\sum_{t \in \Tl} \sum_c \Pr(c) \left\langle p_{t,c}, \E\left[\widetilde{\ell}_{t,c}\big| \His_{t-1}\right] - \widetilde{\ell}_{t,c} \right\rangle L_t.\] 

Since the probability vector $p_{t,c}$ is determined at time $t-1$, the indicator function $L_t$ is also determined at time $t-1$. Thus the summand random variable \[\sum_c \Pr(c) \left\langle p_{t,c}, \E\left[\widetilde{\ell}_{t,c}\big| \His_{t-1}\right] - \widetilde{\ell}_{t,c} \right\rangle L_t\] satisfies
\begin{align*}
    &\E\left[ \sum_c \Pr(c) \left\langle p_{t,c}, \E\left[\widetilde{\ell}_{t,c}\big| \His_{t-1}\right] - \widetilde{\ell}_{t,c} \right\rangle L_t \big| \His_{t-1} \right]\\
    =& \E\left[ \sum_c \Pr(c) \left\langle p_{t,c}, \E\left[\widetilde{\ell}_{t,c}\big| \His_{t-1}\right] - \widetilde{\ell}_{t,c} \right\rangle \big| \His_{t-1} \right] L_t\\
    =& 0.
\end{align*}
Thus the sequence of random varialbes \[\left\{ \sum_c \Pr(c) \left\langle p_{t,c}, \E\left[\widetilde{\ell}_{t,c}\big| \His_{t-1}\right] - \widetilde{\ell}_{t,c} \right\rangle L_t \right\}_{t \in \Tl}\] forms a martingale difference sequence with respect to the filtration $\{ \His_t \}_{t \in \Tl}$.

The summand random variable further satisfies
\begin{align*}
     & \sum_c \Pr(c) \left\langle p_{t,c},  \widetilde{\ell}_{t,c} \right\rangle L_t \\
     =& \sum_c \Pr(c) p_{t,c}(a_t) \frac{\ell_{t,c}(a_t)}{f_e(a_t) + \gamma} L_t \\
     \le& \sum_c \Pr(c) p_{t,c}(a_t) \frac{1}{f_e(a_t) + \gamma} L_t \\
     =&   \frac{p_t(a_t)}{f_e(a_t) + \gamma} L_t\\
     \le& 4.
\end{align*}

Then by Azuma-Hoeffding inequality, the following inequality holds with probability at least $1- \delta$:
\[\sum_{t \in \Tl} \sum_c \Pr(c) \left\langle p_{t,c}, \E\left[\widetilde{\ell}_{t,c}\big| \His_{t-1}\right] - \widetilde{\ell}_{t,c} \right\rangle L_t \le 4 \sqrt{T \log(\frac{1}{\delta})}.\]
Thus we have
\[\sum_{t \in \Tl} \sum_c \Pr(c) \left\langle p_{t,c}, \E\left[\widetilde{\ell}_{t,c}\big| \His_{t-1}\right] - \widetilde{\ell}_{t,c} \right\rangle \le  4 \sqrt{T \log(\frac{1}{\delta})}\]
with probability at least $\Pr(G) - \delta$.

Combining the first component and the second component, we get that
\[ \sum_{t \in \Tl}  \sum_{c} \Pr(c) \left\langle p_{t,c}, \ell_{t,c} - \widetilde{\ell}_{t,c} \right\rangle \le 4 \sqrt{T \log(\frac{1}{\delta})} + 2KT\gamma \]
with probability at least $\Pr(G) - \delta$.

\subsection{Combining the pieces}
Combining all previous arguments, we get that
\begin{align*}
    &\Reg(\pi)\\
    \le&  2 \sqrt{T \log(\frac{1}{\delta})} \\
    +& 4 \sqrt{T \log(\frac{1}{\delta})} \\
    +& 2 \frac{1}{\eta} \log K + 8 \eta \left( 2 KT + \frac{1}{\gamma} \log(\frac{1}{\delta}) \right) \\
    +& 8 \sqrt{T \log(\frac{1}{\delta})} + 4KT\gamma\\
    +& 8 K \gamma T +  2\frac{K \exp(-\iota) T}{\gamma} + 2\frac{98 K T \iota}{L} + 2\frac{\gamma^2 L K T}{\iota} \\
    +& 16\left(\sqrt{2KT\iota \log(\frac{\delta}{2})} + 2\sqrt{2\frac{TK^2}{L}\iota \log(\frac{\delta}{2})}\right)\\
    +&  30 \sqrt{T\log(\frac{1}{\delta})}\\
    +& 2(\frac{8 \cdot 36 \iota}{\gamma} + 2\frac{8}{\gamma}) \log(1/\delta) + 2KT \exp(-\iota)
\end{align*}
with probability at least $\Pr(G) - 8\delta$.

Taking $\iota=2 \log (8 K T \frac{1}{\delta})$, $L=\sqrt{\frac{\iota K T}{\log (K)}}=\widetilde{\Theta}(\sqrt{K T \log\frac{1}{\delta}} )$, $\gamma=\frac{16 \iota}{L}=\widetilde{\Theta}(\sqrt{\frac{\log(1/\delta)}{K T}})$, and $\eta=\frac{\gamma}{2(2 L \gamma+\iota)}=\widetilde{\Theta}(1 / \sqrt{K T \log(1/\delta)})$, it is easy to see that $\Reg(\pi)= \widetilde{O}(\sqrt{KT\log\frac{1}{\delta}})$ with probability at least $\Pr(G) - 8\delta \ge 1 - 9\delta$  for any policy $\pi$ and any $\delta \in (0,1)$. 

The final step is rescaling the probability constant by a factor of $1/9$, which gives that $\Reg(\pi)= \widetilde{O}(\sqrt{KT\log\frac{1}{\delta}})$ with probability at least $1 - \delta$ and ends the proof of \Cref{thm:high_probability_main}.

%% file: main.bbl
\begin{thebibliography}{18}
\providecommand{\natexlab}[1]{#1}
\providecommand{\url}[1]{\texttt{#1}}
\expandafter\ifx\csname urlstyle\endcsname\relax
  \providecommand{\doi}[1]{doi: #1}\else
  \providecommand{\doi}{doi: \begingroup \urlstyle{rm}\Url}\fi

\bibitem[Balseiro et~al.(2019)Balseiro, Golrezaei, Mahdian, Mirrokni, and Schneider]{balseiro2019contextual}
Santiago Balseiro, Negin Golrezaei, Mohammad Mahdian, Vahab Mirrokni, and Jon Schneider.
\newblock Contextual bandits with cross-learning.
\newblock \emph{Advances in Neural Information Processing Systems}, 32, 2019.

\bibitem[Schneider and Zimmert(2023)]{Zimmert2023}
Jon Schneider and Julian Zimmert.
\newblock Optimal cross-learning for contextual bandits with unknown context distributions.
\newblock In A.~Oh, T.~Naumann, A.~Globerson, K.~Saenko, M.~Hardt, and S.~Levine, editors, \emph{Advances in Neural Information Processing Systems}, volume~36, pages 51862--51880. Curran Associates, Inc., 2023.

\bibitem[Li et~al.(2010)Li, Chu, Langford, and Schapire]{background_Schapire_2010}
Lihong Li, Wei Chu, John Langford, and Robert~E. Schapire.
\newblock A contextual-bandit approach to personalized news article recommendation.
\newblock In \emph{Proceedings of the 19th International Conference on World Wide Web}, WWW '10, page 661–670, New York, NY, USA, 2010. Association for Computing Machinery.
\newblock ISBN 9781605587998.
\newblock \doi{10.1145/1772690.1772758}.

\bibitem[Kale et~al.(2010)Kale, Reyzin, and Schapire]{backgroundref_Kale_2010}
Satyen Kale, Lev Reyzin, and Robert~E Schapire.
\newblock Non-stochastic bandit slate problems.
\newblock In J.~Lafferty, C.~Williams, J.~Shawe-Taylor, R.~Zemel, and A.~Culotta, editors, \emph{Advances in Neural Information Processing Systems}, volume~23. Curran Associates, Inc., 2010.

\bibitem[Villar et~al.(2015)Villar, Bowden, and Wason]{background_villar_2015}
Sof{\'i}a~S. Villar, Jack Bowden, and James Wason.
\newblock {Multi-armed Bandit Models for the Optimal Design of Clinical Trials: Benefits and Challenges}.
\newblock \emph{Statistical Science}, 30\penalty0 (2):\penalty0 199 -- 215, 2015.
\newblock \doi{10.1214/14-STS504}.

\bibitem[Han et~al.(2020{\natexlab{a}})Han, Zhou, and Weissman]{Han2020Stochastic}
Yanjun Han, Zhengyuan Zhou, and Tsachy Weissman.
\newblock Optimal no-regret learning in repeated first-price auctions.
\newblock \emph{ArXiv}, abs/2003.09795, 2020{\natexlab{a}}.

\bibitem[Han et~al.(2020{\natexlab{b}})Han, Zhou, Flores, Ordentlich, and Weissman]{Han2020Adversarial}
Yanjun Han, Zhengyuan Zhou, Aaron Flores, Erik Ordentlich, and Tsachy Weissman.
\newblock Learning to bid optimally and efficiently in adversarial first-price auctions.
\newblock \emph{ArXiv}, abs/2007.04568, 2020{\natexlab{b}}.

\bibitem[Ai et~al.(2022)Ai, Wang, Li, Zhang, Huang, and Deng]{Aiauction}
Rui Ai, Chang Wang, Chenchen Li, Jinshan Zhang, Wenhan Huang, and Xiaotie Deng.
\newblock No-regret learning in repeated first-price auctions with budget constraints, 2022.

\bibitem[Wang et~al.(2023)Wang, Yang, Deng, and Kong]{WangAuction}
Qian Wang, Zongjun Yang, Xiaotie Deng, and Yuqing Kong.
\newblock Learning to bid in repeated first-price auctions with budgets.
\newblock In Andreas Krause, Emma Brunskill, Kyunghyun Cho, Barbara Engelhardt, Sivan Sabato, and Jonathan Scarlett, editors, \emph{Proceedings of the 40th International Conference on Machine Learning}, volume 202 of \emph{Proceedings of Machine Learning Research}, pages 36494--36513. PMLR, 23--29 Jul 2023.

\bibitem[Kleinberg et~al.(2010)Kleinberg, Niculescu-Mizil, and Sharma]{Kleinberg2010}
Robert~D. Kleinberg, Alexandru Niculescu-Mizil, and Yogeshwer Sharma.
\newblock Regret bounds for sleeping experts and bandits.
\newblock \emph{Machine Learning}, 80:\penalty0 245--272, 2010.

\bibitem[Neu and Valko(2014)]{Neu2014}
Gergely Neu and Michal Valko.
\newblock Online combinatorial optimization with stochastic decision sets and adversarial losses.
\newblock In Z.~Ghahramani, M.~Welling, C.~Cortes, N.~Lawrence, and K.Q. Weinberger, editors, \emph{Advances in Neural Information Processing Systems}, volume~27. Curran Associates, Inc., 2014.

\bibitem[Kale et~al.(2016)Kale, Lee, and P{\'{a}}l]{Kale16}
Satyen Kale, Chansoo Lee, and D{\'{a}}vid P{\'{a}}l.
\newblock Hardness of online sleeping combinatorial optimization problems.
\newblock In Daniel~D. Lee, Masashi Sugiyama, Ulrike von Luxburg, Isabelle Guyon, and Roman Garnett, editors, \emph{Advances in Neural Information Processing Systems 29: Annual Conference on Neural Information Processing Systems 2016, December 5-10, 2016, Barcelona, Spain}, pages 2181--2189, 2016.

\bibitem[Saha et~al.(2020)Saha, Gaillard, and Valko]{Saha2020}
Aadirupa Saha, Pierre Gaillard, and Michal Valko.
\newblock Improved sleeping bandits with stochastic actions sets and adversarial rewards.
\newblock In \emph{Proceedings of the 37th International Conference on Machine Learning}, ICML'20. JMLR.org, 2020.

\bibitem[Neu and Olkhovskaya(2020)]{Neu2020contextual}
Gergely Neu and Julia Olkhovskaya.
\newblock Efficient and robust algorithms for adversarial linear contextual bandits.
\newblock In Jacob Abernethy and Shivani Agarwal, editors, \emph{Proceedings of Thirty Third Conference on Learning Theory}, volume 125 of \emph{Proceedings of Machine Learning Research}, pages 3049--3068. PMLR, 09--12 Jul 2020.

\bibitem[Dai et~al.(2023)Dai, Luo, Wei, and Zimmert]{Dai2023}
Yan Dai, Haipeng Luo, Chen{-}Yu Wei, and Julian Zimmert.
\newblock Refined regret for adversarial mdps with linear function approximation.
\newblock In Andreas Krause, Emma Brunskill, Kyunghyun Cho, Barbara Engelhardt, Sivan Sabato, and Jonathan Scarlett, editors, \emph{International Conference on Machine Learning, {ICML} 2023, 23-29 July 2023, Honolulu, Hawaii, {USA}}, volume 202 of \emph{Proceedings of Machine Learning Research}, pages 6726--6759. {PMLR}, 2023.

\bibitem[Syrgkanis et~al.(2016)Syrgkanis, Luo, Krishnamurthy, and Schapire]{Luo2016}
Vasilis Syrgkanis, Haipeng Luo, Akshay Krishnamurthy, and Robert~E. Schapire.
\newblock Improved regret bounds for oracle-based adversarial contextual bandits.
\newblock In Daniel~D. Lee, Masashi Sugiyama, Ulrike von Luxburg, Isabelle Guyon, and Roman Garnett, editors, \emph{Advances in Neural Information Processing Systems 29: Annual Conference on Neural Information Processing Systems 2016, December 5-10, 2016, Barcelona, Spain}, pages 3135--3143, 2016.

\bibitem[Liu et~al.(2023)Liu, Wei, and Zimmert]{ZimmertBypass}
Haolin Liu, Chen-Yu Wei, and Julian Zimmert.
\newblock Bypassing the simulator: Near-optimal adversarial linear contextual bandits.
\newblock In A.~Oh, T.~Naumann, A.~Globerson, K.~Saenko, M.~Hardt, and S.~Levine, editors, \emph{Advances in Neural Information Processing Systems}, volume~36, pages 52086--52131. Curran Associates, Inc., 2023.

\bibitem[Neu(2015)]{EXP3-IX}
Gergely Neu.
\newblock Explore no more: Improved high-probability regret bounds for non-stochastic bandits.
\newblock In Corinna Cortes, Neil~D. Lawrence, Daniel~D. Lee, Masashi Sugiyama, and Roman Garnett, editors, \emph{Advances in Neural Information Processing Systems 28: Annual Conference on Neural Information Processing Systems 2015, December 7-12, 2015, Montreal, Quebec, Canada}, pages 3168--3176, 2015.

\end{thebibliography}
